\documentclass[11pt]{article} 
\usepackage{url}
\usepackage{smile}
\usepackage{graphicx} 
\usepackage{algorithm}
\usepackage{algorithmic}
\usepackage{todonotes}
\usepackage{epstopdf}
\usepackage{wrapfig}
\usepackage[colorlinks, linkcolor=blue, anchorcolor=blue, citecolor=blue]{hyperref}
\usepackage[margin=1in]{geometry}
\usepackage[normalem]{ulem}
\usepackage[export]{adjustbox}
\usepackage{mathtools, cuted}
\usepackage{natbib}
\usepackage{enumerate}
\usepackage{enumitem}
\usepackage{kpfonts}
\DeclareMathAlphabet{\mathsf}{OT1}{cmss}{m}{n}

\SetMathAlphabet{\mathsf}{bold}{OT1}{cmss}{bx}{n}

\providecommand{\norm}[1]{\|#1\|}

\newtheorem*{theorem*}{Theorem}

\title{\huge \bf On Generalization Bounds of a Family of Recurrent Neural Networks\footnote{Presented in NeurIPS Workshop on Integration of Deep Learning Theories, 2018.}}

\author{Minshuo Chen, Xingguo Li, Tuo Zhao\thanks{Minshuo Chen and Tuo Zhao are affiliated with School of Industrial and Systems Engineering at Georgia Tech; Xingguo Li is affiliated with Computer Science Department at Princeton University; Email:$\{$mchen393, tourzhao$\}$@gatech.edu.}}

\newcommand{\commentout}[1]{}

\begin{document}

\maketitle


\begin{abstract}
Recurrent Neural Networks (RNNs) have been widely applied to sequential data analysis. Due to their complicated modeling structures, however, the theory behind is still largely missing. To connect theory and practice, we study the generalization properties of vanilla RNNs as well as their variants, including Minimal Gated Unit (MGU), Long Short Term Memory (LSTM), and Convolutional (Conv) RNNs. Specifically, our theory is established under the PAC-Learning framework. The generalization bound is presented in terms of the spectral norms of the weight matrices and the total number of parameters. We also establish refined generalization bounds with additional norm assumptions, and draw a comparison among these bounds. We remark: (1) Our generalization bound for vanilla RNNs is significantly tighter than the best of existing results; (2) We are not aware of any other generalization bounds for MGU, LSTM, and Conv RNNs in the exiting literature; (3) We demonstrate the advantages of these variants in generalization.
\end{abstract}


\section{Introduction}\label{sec:intro}

Recurrent Neural Networks (RNNs) have successfully revolutionized sequential data analysis, and been widely applied to many real world problems, such as natural language processing \citep{cho2014learning, bahdanau2014neural, sutskever2014sequence}, speech recognition \citep{graves2006connectionist, mikolov2010recurrent, graves2012sequence, graves2013speech}, computer vision \citep{gregor2015draw, xu2015show, donahue2015long, karpathy2015deep}, healthcare \citep{lipton2015learning, choi2016doctor, choi2016using}, and robot control \citep{lee2000identification, yoo2006adaptive}. Quite a few of these applications can be approached easily in our daily life, such as Google Translate, Google Now, Apple Siri, etc.

The sequential modeling nature of RNNs is significantly different from feedforward neural networks, though they both have neurons as the basic components. RNNs exploit the internal state (also known as hidden unit) to process the sequence of inputs, which naturally captures the dependence of the sequence. 
Besides the vanilla version, RNNs have many other variants. A large class of variants incorporate the so-called ``gated'' units to trim RNNs for different tasks. Typical examples include Long Short-Term Memory (LSTM, \cite{hochreiter1997long}), Gated Recurrent Unit (GRU, \cite{jozefowicz2015empirical}) and Minimal Gated Unit (MGU, \cite{zhou2016minimal}). 

The success of RNNs owes not only to their special network structures and the ability to fit data, but also to their good generalization property: They provide accurate predictions on unseen data. For example, \cite{graves2013speech} report that after training with merely 462 speech samples, deep LSTM RNNs achieve a test set error of $17.7\%$ on TIMIT phoneme recognition benchmark, which is the best recorded score. 
Despite of the popularity of RNNs in applications, their theory is less studied than other feedforward neural networks \citep{haussler1992decision, bartlett2017spectrally, neyshabur2017pac, golowich2017size, li2018tighter}. There are still several long lasting fundamental questions regarding the approximation, trainability, and generalization of RNNs.

In this paper, we propose to understand the generalization ability of RNNs and their variants. We aim to answer two questions from a theoretical perspective:

\noindent\emph{\textbf{Q.1) Do RNNs suffer from significant curse of dimensionality?}}

\noindent\emph{\textbf{Q.2) What are the advantages of MGU and LSTM over vanilla RNNs?}}

The investigation of generalization properties of RNNs has a long history. 
Many early works are based on oversimplified assumptions. \cite{dasgupta1996sample} and \cite{koiran1998vapnik}, for example, adopt a VC-dimension argument to show complexity bounds of RNNs that are polynomial in the size of the network. They, however, either consider linear RNNs for binary classification tasks, or assume RNNs take the first coordinate of their hidden states as outputs. 
More recently, \cite{bartlett2017spectrally} propose a new technique for developing generalization bounds for feedforward neural networks based on empirical Rademacher complexity under the PAC-Learning framework. \cite{neyshabur2017pac} further adapt the technique to establish their generalization bound using the PAC-Bayes approach. The follow-up work \cite{zhang2018stabilizing} use the PAC-Bayes approach to establish a generalization bound for vanilla RNNs.

We exploit the compositional nature of RNNs, and decouples the spectral norms of weight matrices and the number of weight parameters. This makes our analysis conceptually much simpler (e.g. avoid layer wise analysis), and also yields better generalization bound than \cite{zhang2018stabilizing}. 

\begin{wrapfigure}{r}{0.24\textwidth}
\vspace{-0.25in}
\begin{center}
\includegraphics[width=0.22\textwidth]{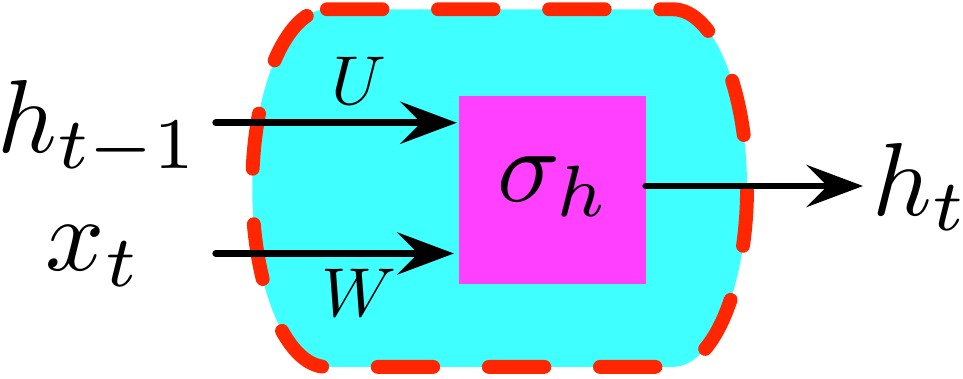}
\end{center}
\vspace{-0.18in}
\caption{A building block of vanilla RNNs.}
\vspace{-0.2in}
\end{wrapfigure}

Consider vanilla RNNs, we observe $m$ sequences of data points $\left(x_{i,t},z_{i,t}\right)_{t=1}^T$, where $x_{i,t}\in\RR^{d_x}$ and the response $z_{i,t} \in \cZ$ for all $t=1,...,T$ and $i=1,...,m$. Each sequence is drawn independently from some underlying distribution over $\RR^{d_x \times T} \times \cZ$. Extensions to dependent sequences are discussed in Section \ref{sec:discussion}, however, note that data points $(x_{i, t}, z_{i, t})$ can be dependent within a sequence, i.e., for a fixed $i \in \{1, \dots, m\}$. The vanilla RNNs compute $h_{i, t}$ and $y_{i, t}$ iteratively as follows,
\begin{align*}
h_{i, t} = \sigma_h\left(U h_{i, t-1}+Wx_{i, t}\right),\quad\textrm{and}\quad y_{i, t} = \sigma_y\left(Vh_{i, t}\right),
\end{align*}
where $\sigma_y$ and $\sigma_h$ are activation operators, $h_{i, t} \in \RR^{d_h}$ is the hidden state with $h_{i, 0} = 0$, and $U \in \RR^{d_h \times d_h}$, $V \in \RR^{d_y \times d_h}$, and $W \in \RR^{d_h \times d_x}$ are weight matrices. The activation operators $\sigma_h$ and $\sigma_y$ are entrywise, i.e., $\sigma_h(\left[v_1, \dots, v_d\right]^\top) = [\sigma_h(v_1), \dots, \sigma_h(v_d)]^\top$, and Lipschitz with parameters $\rho_h$ and $\rho_y$ respectively. We assume $\sigma_h(\cdot) = \tanh(\cdot)$, $\sigma_y(0) = 0$, and $\rho_y = 1$. Extensions to general activations are given in Section \ref{sec:theory}.

\textbf{Our Contribution}. To establish the generalization bound, we need to define the ``model complexity" of vanilla RNNs. In this paper, we adopt the empirical Rademacher complexity (ERC, see more details in Section \ref{sec:theory}), which has been widely used in the existing literature on PAC-Learning. 
For many nonparametric function classes, we often need complicated argument to upper bound their ERC. Our analysis, however, shows that we can upper bound the ERC of vanilla RNNs in a very simple manner by exploiting their Lipschitz continuity with respect to the model parameters, since they are essentially in parametric forms. More specifically, 
denote $\cF_t = \{f_t : \{x_1,...,x_t\} \mapsto y_t\}$ as the class of mappings from the first $t$ inputs to the $t$-th output computed by vanilla RNNs. For a matrix $A$, $\lVert A \rVert_2$ denotes the spectral norm, and for a vector $v$, $\lVert v \rVert_2$ denotes the Euclidean norm. Define $\frac{a^t-1}{a-1} = t$ for $a = 1$. 
Then, informally speaking, the ``model complexity'' of vanilla RNNs satisfies
\begin{align}
\textsf{Complexity} = O\Bigg(d \lVert V \rVert_2 \min\big\{\sqrt{d}, \lVert W \rVert_2 \frac{\lVert U \rVert_2^t - 1}{\lVert U \rVert_2 - 1}\big\} \times \sqrt{\log \left(t \sqrt{d} \frac{\lVert U \rVert_2^t - 1}{\lVert U \rVert_2 - 1} \right)} \Bigg), \notag
\end{align}
where $d = \sqrt{d_xd_h + d_h^2 + d_hd_y}$. 

We then consider a new testing sequence $\left(x_{t},z_{t}\right)_{t=1}^T$. The response sequence is computed by $\tilde{z}_t = \phi(y_{t}), \quad \text{for all}~ t = 1, \dots, T,$
where $\phi$ is a function mapping the output of vanilla RNNs to the response of interest. In practice, the function $\phi$ varies across different data analysis tasks. For example, in sequence to sequence classification, we take $\phi(y_t) = \argmax_j [y_t]_j$; in regression, we take $\phi(y_t) = y_t$; in density estimation, we can take $\phi(y_t) = \textrm{softmax}(y_t)$. 

We further define a risk function that can unify different data analysis tasks. Specifically, let $\cL(\cA(y, z))$ be a loss function, where $\cA(y, z)$ is a function taking the output $y_t$ and the observed response $z_t$ as inputs, and $\cL$ is chosen according to different tasks. Then we define the population risk for the $t$-th output as $\cR(f_t) = \EE[\cL(\cA(y_t, z_t))]$. Its empirical counterpart is similarly defined as $\hat{\cR}(f_t) = \frac{1}{m} \sum_{i=1}^m \cL(\cA(y_{i, t}, z_{i, t}))$. Training RNNs is essentially minimizing the empirical risk $\hat{\cR}(f_t)$. Many applications can be formulated into this framework. For example, in classification, we take $\cA = -\cM$ as the functional margin operator and $\cL = \cL_\gamma$ as the ramp loss with $\gamma$ being the margin value (see detailed definitions in Section \ref{sec:theory}); in regression, we take $\cA(y_t, z_t) = y_t - z_t$ and $\cL$ as the $\ell_p$ loss for $p \in \ZZ_+$. 
We then give the generalization bound in the following statement.
\begin{theorem}[informal]\label{thm:boundgeneral}
Assume the input data space is bounded, i.e., $\norm{x}_2\leq 1$ and $z \in \cZ$ bounded. Suppose the mapping $\cA(y, z)$ is Lipschitz in $y$, and the loss function $\cL$ satisfies $|\cL(\cA(y, z))| \leq B$ and is $L$-Lipschitz for any $y$ computed by RNNs and $z \in \cZ$. Given a collection of samples $S = \left\{(x_{i,t},z_{i,t})_{t=1}^T, i = 1,...,m\right\}$ and a new testing sequence $(x_t, z_t)_{t=1}^T$, with probability at least $1-\delta$ over $S$, for any $f_t \in \cF_t$ with integer $t \leq T$, we have,
\begin{align}
\cR(f_t) \leq \hat{\cR}(f_t) + \tilde{O}\bigg(\frac{L \times \textsf{Complexity}}{\sqrt{m}} + B \sqrt{\frac{\log (1 / \delta)}{m}} \bigg). \notag 
\end{align}
\end{theorem}
Please refer to Section \ref{sec:theory} for a complete statement. Most of the aforementioned commonly used $\cA$ and $\cL$ satisfy the assumptions in Theorem \ref{thm:boundgeneral}. For example, in classification, the functional margin operator $\cM(y, z)$ is $2$-Lipschitz in $\tilde{z}$. The ramp loss $\cL_\gamma$ is uniformly bounded by $1$ and $\frac{1}{\gamma}$-Lipschitz. In regression, $\cA(y, z)$ is $1$-Lipschitz in $y$ and bounded since the input data are bounded. Then the $\ell_p$ loss becomes bounded and Lipschitz due to its bounded input.

\textbf{Comparison with Existing Results}. To better understand the obtained generalization bound and draw a comparison among existing literature, we instantiate Theorem \ref{thm:boundgeneral} for sequence to sequence classification using vanilla RNNs. Recall that for classification tasks, we have $L = 1 / \gamma$, $B = 1$ and $\cM(y, z)$ is $2$-Lipschitz in $y$. We list the corresponding generalization bounds in Table \ref{table:genbound} according to the magnitude
of $\norm{U}_2$.
\begin{table}[h]
\centering
\caption{Generalization bounds for vanilla RNNs in classification tasks (we only list the order of the gap $\cR(f_t) - \hat{\cR}(f_t)$). The third column lists the result obtained in \cite{zhang2018stabilizing}.}
\label{table:genbound}
\begin{tabular}{c | c | c}
\hline
& Theorem \ref{thm:boundgeneral} & \cite{zhang2018stabilizing} \\\hline
&& \\[-1em]
\textbf{(I)} $\norm{U}_2 < 1$ & $\tilde{O}\big(d/ \sqrt{m}\gamma \big)$ & $\tilde{O}\big(dt^2 / \sqrt{m}\gamma \big)$ \\\hline
&& \\[-1em]
\textbf{(II)} $\lVert U \rVert_2 = 1$ & $\tilde{O}\big(dt / \sqrt{m}\gamma \big)$ & $\tilde{O}\big(dt^2 / \sqrt{m}\gamma \big)$ \\\hline
&& \\[-1em]
\textbf{(III)} $\lVert U \rVert_2 > 1$ & $\tilde{O}\big(\sqrt{d^3 t} / \sqrt{m}\gamma \big)$ & $\tilde{O}\big(dt^2 \lVert U \rVert_2^{t} / \sqrt{m}\gamma \big)$ \\\hline
\end{tabular}
\end{table}

As can be seen, the obtained generalization bound only has a polynomial dependence on the size of vanilla RNNs,  i.e., width $d$ and sequence length $t$. Thus, we theoretically justify that the complexity of vanilla RNNs do not suffer from significant curse of dimensionality. Because they compute outputs $y_t$ recursively using the same weight matrices, and their hidden states $h_t$ are entrywise bounded.

We compare Theorem \ref{thm:boundgeneral} with the generalization bound obtained in \cite{zhang2018stabilizing}, which is of the order
$\tilde{O}\left(dt^2\lVert W \rVert_2 \lVert V \rVert_2 \max\{1, \lVert U \rVert_2^{t}\}/ \sqrt{m}\gamma\right),$
and we distinguish the same three different scenarios as listed in Table \ref{table:genbound}. Our bound is tighter by a factor of $t^2$ for case \textbf{(I)}, a factor of $t$ for case \textbf{(II)}. Additionally, \cite{zhang2018stabilizing} fail to incorporate the boundedness condition of hidden state into their analysis, thus the generalization bound is exponential in $t$ for case \textbf{(III)}. Our generalization bound, however, is still polynomial in $d$ and $t$ for case \textbf{(III)}.

Moreover, \textbf{(II)} is closely related to a few recent results on imposing orthogonal constraints on weight matrices to stabilize the training of RNNs \citep{saxe2013exact, le2015simple, arjovsky2016unitary, vorontsov2017orthogonality, zhang2018stabilizing}. We remark that from a learning theory perspective, \textbf{(II)} implies that orthogonal constraints can potentially help generalization.


%
We also present refined generalization bounds with additional matrix norm assumptions. These assumptions allow us to derive norm-based generalization bounds. We draw a comparison among these bounds and highlight their advantage under different scenarios.

Our theory can be further extended to several variants, including MGU and LSTM RNNs, and convolutional RNNs (Conv RNNs). Specifically, we show that the gated units in MGU and LSTM RNNs can introduce extra decaying factors to further reduce the dependence on $d$ and $t$ in generalization. The convolutional filters in Conv RNNs can reduce the dependence on $d$ through parameter sharing. Such an advantage in generalization makes these RNNs do not suffer from significant curse of dimensionality. To the best of our knowledge, these are the first results on generalization guarantees for these neural networks.


\noindent\textbf{Notations}: Given a vector $v \in \RR^d$, we denote its Euclidean norm by $\lVert v \rVert_2^2 = \sum_{i=1}^d |v_i|^2$, and the infinity norm by $\lVert v \rVert_\infty = \max_j |v_j|$. Given a matrix $M \in \RR^{m \times n}$, we denote the spectral norm by $\lVert M \rVert_2$ as the largest singular value of $M$, the Frobenius norm by $\lVert M \rVert_{\textrm{F}}^2 = \textrm{trace}(M M^\top)$, and the $(2, 1)$ norm by $\norm{M}_{2, 1} = \sum_{i=1}^n \norm{M_{:, i}}_2$. Given a function $f$, we denote the function infinity norm by $\lVert f \rVert_\infty = \sup |f|$. We use $\tilde{O}(\cdot)$ to denote $O(\cdot)$ with hidden log factors.







\section{Generalization of Vanilla RNNs} \label{sec:theory}

To establish the generalization bound, we start with imposing some mild assumptions.
\begin{assumption}\label{assump1}
Input data are bounded, i.e., $\lVert x_{i, t} \rVert_2 \leq B_x$ for all $i = 1, \dots, m$ and $t = 1, \dots, T$.
\end{assumption}

\begin{assumption}\label{assump2}
The spectral norms of weight matrices are bounded respectively, i.e., $\lVert U \rVert_2 \leq B_U$, $\lVert V \rVert_2 \leq B_V,$ and $\lVert W \rVert_2 \leq B_W.$
\end{assumption}

\begin{assumption}\label{assump3}
Activation operators $\sigma_h$ and $\sigma_y$ are Lipschitz with parameters $\rho_h$ and $\rho_y$ respectively, and $\sigma_h(0) = \sigma_y(0) = 0$. Additionally, $\sigma_h$ is entrywise bounded by $b$.
\end{assumption}
Assumptions \ref{assump1} and \ref{assump2} are moderate assumptions. Moreover, Assumption \ref{assump3} holds for most commonly used activation operators, such as $\sigma_h(\cdot) = \tanh(\cdot)$ and $\sigma_y(\cdot) = \textrm{ReLU}(\cdot) = \max\{\cdot, 0\}$ (1-Lipschitz). 


Recall vanilla RNNs compute $h_{i, t}$ and $y_{i, t}$ as follows, $$h_{i, t} = \sigma_h\left(U h_{i, t-1}+Wx_{i, t}\right)\quad\textrm{and}\quad y_{i, t} = \sigma_y\left(Vh_{i, t}\right),$$ where $U \in \RR^{d_h \times d_h}$, $V \in \RR^{d_y \times d_h}$, and $W \in \RR^{d_h \times d_x}$. We consider multiclass classification tasks with the label $z \in \cZ = \{1, \dots, K\}$. Given a sequence $(x_t, z_t)_{t=1}^T$, we define $X_t \in \RR^{d_x \times t}$ by concatenating $x_1, \dots, x_t$ as columns of $X_t$. Recall that we denote $\cF_t = \{f_t : X_t \mapsto y_t\}$ as the class of mappings from the first $t$ inputs to the $t$-th output computed by vanilla RNNs.

As previously mentioned, we define the functional margin for the $t$-th output in vanilla RNNs as
\begin{align*}
\textstyle \cM(f_t(X_t), z_t) = [f_t(X_t)]_{z_t} - \max_{j \neq z_t} [f_t(X_t)]_j.
\end{align*}
We further define a ramp loss $\cL_\gamma\left(-\cM(f_t(X_t), z_t)\right) : \RR \mapsto \RR^+$ to each margin, where $\cL_\gamma$ is a piecewise linear function defined as 
\begin{align*}
\cL_\gamma(a) = \mathds{1}\{a > 0\} + (1+ a / \gamma)\mathds{1}\{-\gamma \leq a \leq 0\},
\end{align*}
where $\mathds{1}\{A\}$ denotes the indicator function of a set $A$. Accordingly, the ramp risk is defined as $$\cR_\gamma(f_t) = \EE \left[\cL_\gamma\left(-\cM(f_t(X_t), z_t)\right)\right],$$ and its empirical counterpart is defined as $$\hat{\cR}_\gamma(f_t) = \frac{1}{m} \sum_{i=1}^m \cL_\gamma\left(-\cM(f_t(X_{i, t}), z_{i, t})\right).$$ We then present the formal statement of Theorem \ref{thm:boundgeneral}.
\begin{theorem}\label{thm:genbound}
Let activation operators $\sigma_h$ and $\sigma_y$ be given, and Assumptions \ref{assump1}--\ref{assump3} hold. Then for $(x_t, z_t)_{t=1}^T$ and $S = \left\{(x_{i, t}, z_{i, t})_{t=1}^T, i = 1, \dots, m \right\}$ drawn i.i.d. from any underlying distribution over $\RR^{d_x \times T} \times \{1, \dots, K\}$, with probability at least $1 - \delta$ over $S$, for every margin value $\gamma > 0$, sufficiently large sample size $m$, and every $f_t \in \cF_t$ for integer $t \leq T$, we have
\begin{align}
\PP\left(\tilde{z}_t \neq z_t \right) \leq \hat{\cR}_\gamma(f_t) + 3 \sqrt{\frac{\log \frac{2}{\delta}}{2m}} + O \left(\frac{d \rho_y B_V \lambda_t \sqrt{\log \big(t \sqrt{dm} \frac{\beta^t - 1}{\beta - 1}\big)}}{\sqrt{m}\gamma}\right), \label{eq:genbound}
\end{align}
where $d = \sqrt{d_xd_h+d_h^2+ d_hd_y}$, $\beta = \rho_h B_U$, and $\lambda_t = \min\big\{b \sqrt{d}, \rho_h B_x B_W \frac{\beta^t - 1}{\beta - 1}\big\}$.
\end{theorem}

\begin{remark}
To ease the presentation, we only provide the generalization bound for the classification task. Extensions to general tasks are straightforward by replacing functions $\cA$ and $\cL$ and substituting suitable values of $L$ and $B$.
\end{remark}

The generalization bound depends on the total number of weights, and the range of $\rho_h B_U$ in three cases as indicated in Section \ref{sec:intro}. More precisely, if $\rho_h B_U \lesssim (1 + \frac{1}{t^\alpha})$ for constant $\alpha>0$ bounded away from zero, the generalization bound is of the order $\tilde{O}\left(\frac{dt^{\alpha}}{\sqrt{m}\gamma}\right)$, which has a polynomial dependence on $d$ and $t$. As can be seen, with proper normalization on model parameters, the model complexity of vanilla RNNs do not suffer from significant curse of dimensionality.

We also highlight a tradeoff between generalization and representation of vanilla RNNs. As can be seen, when $\rho_h B_U$ is strictly smaller than $1$, the generalization bound is nearly independent on $t$. The hidden state, however, only has limited representation ability, since its magnitude diminishes as $t$ grows large. On the contrary, when $\rho_hB_U$ is strictly greater than $1$, the representation ability is amplified but the generalization becomes worse. As a consequence, recent empirical results show that imposing extra constraints or regularization, such as $U^\top U = I$ or $\lVert U \rVert_2 \leq 1$ \citep{saxe2013exact, le2015simple, arjovsky2016unitary, vorontsov2017orthogonality, zhang2018stabilizing}, helps balance the generalization and representation of RNNs.

\section{Proof of Main Results}\label{sec:proof}

Our analysis is based on the PAC-learning framework. Due to space limit, we only present an outline of our proof. More technical details are deferred to Appendix \ref{pf:sectheory}. Before we proceed, we first define the empirical Rademacher complexity as follows.
\begin{definition}[Empirical Rademacher Complexity]
Let $\cH$ be a function class and $S = \{s_1, \dots, s_m\}$ be a collection of samples. The empirical Rademacher complexity of $\cH$ given $S$ is defined as
\begin{align*}
\mathfrak{R}_S(\cH) = \EE_\epsilon \left[\sup_{h \in \cH} \frac{1}{m} \sum_{i=1}^m \epsilon_i h(s_i)\right],
\end{align*}
where $\epsilon_i$'s are i.i.d. Rademacher random variables, i.e., $\PP(\epsilon_i=1)=\PP(\epsilon_i=-1)=0.5$.
\end{definition}

We then proceed with our analysis. Recall that \cite{mohri2012foundations} give an empirical Rademacher complexity (ERC)-based generalization bound, which is restated in the following lemma with $\cF_{\gamma, t} = \left \{(X_t, z_t) \mapsto \ell_\gamma(-\cM(f_t(X_t), z_t)): f_t \in \cF_t \right\}.$
\begin{lemma}\label{lemma:startpoint}
Given a testing sequence $(x_t, z_t)_{t=1}^T$, with probability at least $1-\delta$ over samples $S = \left\{(x_{i, t}, z_{i, t})_{t=1}^T, i = 1, \dots, m \right\}$, for every margin value $\gamma > 0$ and any $f_t \in \cF_t$, we have
\begin{align}
\PP(\tilde{z}_t\neq z_t)\leq\cR_\gamma(f_t) \leq \hat{\cR}_\gamma(f_t) + 2\mathfrak{R}_S(\cF_{\gamma, t}) + 3 \sqrt{\frac{\log(2/\delta)}{2m}}. \notag 
\end{align}
\end{lemma}
Note that Lemma \ref{lemma:startpoint} adapts the original version (Theorem 3.1, Chapter 3.1, \cite{mohri2012foundations}) for the multiclass ramp loss, and we have $\PP(\tilde{z}_t\neq z_t)\leq \cR_\gamma(f_t)$ by definition. 

Now we only need to bound the ERC $\mathfrak{R}_S(\cF_{\gamma, t})$. Our analysis consists of three steps. First, we characterize the Lipschitz continuity of vanilla RNNs w.r.t model parameters. Next, we bound the covering number of function class $\cF_t$. At last, we derive an upper bound on $\mathfrak{R}_S(\cF_{\gamma, t})$ via the standard machinery in the PAC-learning framework. Specifically, consider two different sets of weight matrices $(U, V, W)$ and $(U', V', W')$. Given the same activation operators and input data, denote the $t$-th output as $y_t$ and $y'_t$ respectively. We characterize the Lipschitz property of $\lVert y_t \rVert_2$ w.r.t model parameters in the following lemma.
\begin{lemma}\label{lemma:ytLipUVW}
Under Assumptions \ref{assump1}--\ref{assump3}, given input $(x_t)_{t=1}^T$ and for any integer $t \leq T$, $\lVert y_t \rVert_2$ is Lipschitz in $U$, $V$ and $W$, i.e.,
\begin{align}
\left \lVert y_t - y'_t \right \rVert_2 \leq L_{U, t} \left \lVert U - U' \right \rVert_\textrm{F} + L_{V, t} \left \lVert V - V' \right \rVert_\textrm{F} + L_{W, t} \left \lVert W - W' \right \rVert_\textrm{F}, \notag
\end{align}
where
$L_{U, t} = \rho_hB_VB_W t a_t$, $L_{V, t} = B_W a_t$, and $L_{W, t} = B_V a_t$ with $a_t = \rho_y\rho_h B_x\frac{(\rho_hB_U)^t - 1}{\rho_h B_U - 1}$.
\end{lemma}
The detailed proof is provided in Appendix \ref{pf:lemmaytLipUVW}. We give a simple example to illustrate the proof technique. Specifically, we consider a single layer network that outputs $y = \sigma(Wx)$, where $x$ is the input, $\sigma$ is an activation operator with Lipschitz parameter $\rho$, and $W$ is a weight matrix. Such a network is Lipschitz in both $x$ and $W$ as follows. Given weight matrices $W$ and $W'$, we have $$\lVert y - y' \rVert_2 = \lVert \sigma(Wx) - \sigma(W' x) \rVert_2 \leq \rho \lVert x \rVert_2 \lVert W - W' \rVert_\textrm{F}.$$ Additionally, given inputs $x$ and $x'$, we have $$\lVert y - y' \rVert_2 = \lVert \sigma(Wx) - \sigma(Wx') \rVert_2 \leq \rho \lVert W \rVert_2 \lVert x - x' \rVert_2.$$ Since vanilla RNNs are multilayer networks, Lemma \ref{lemma:ytLipUVW} can be obtained by telescoping.

We remark that Lemma \ref{lemma:ytLipUVW} is the key to the proof of our generalization bound, which separates the spectral norms of weight matrices and the total number of parameters.

Next, we bound the covering number of $\cF_t$. Denote by $\cN(\cF_t, \epsilon, \textrm{dist}(\cdot, \cdot))$ the minimal cardinality of a subset $\cC \subset \cF_t$ that covers $\cF_t$ at scale $\epsilon$ w.r.t the metric $\textrm{dist}(\cdot, \cdot)$, such that for any $f_t \in \cF_t$, there exists $\hat{f}_t \in \cC$ satisfying $\textrm{dist}(f_t, \hat{f}_t) = \sup_{X_t} \lVert f_t(X_t) - \hat{f}_t(X_t) \rVert_2 \leq \epsilon$. The following lemma gives an upper bound on $\cN(\cF_t, \epsilon, \textrm{dist}(\cdot, \cdot))$.
\begin{lemma}\label{lemma:Ftcovering}
Under Assumptions \ref{assump1}--\ref{assump3}, given any $\epsilon > 0$, the covering number of $\cF_t$ satisfies
\begin{align}
\cN(\cF_t, \epsilon, \textrm{dist}(\cdot, \cdot)) \leq \left(1 + \frac{6 c \sqrt{d} t \rbr{(\rho_hB_U)^t - 1}}{\epsilon\rbr{\rho_h B_U - 1}}\right)^{3d^2}, \notag
\end{align}
where $c = \rho_y\rho_h B_V B_W B_x \max\left\{1, \rho_h B_U\right\}$.
\end{lemma}
The detailed proof is provided in Appendix \ref{pf:lemmaFtcovering}. We briefly explain the proof technique. Given activation operators, since vanilla RNNs are in parametric forms, $f_t$ has a one-to-one correspondence to its weight matrices $U, V$, and $W$. Lemma \ref{lemma:ytLipUVW} implies that $\textrm{dist}(\cdot, \cdot)$ is controlled by the Frobenius norms of the differences of weight matrices. Thus, it suffices to bound the covering numbers of three weight matrices. The product of covering numbers of three weight matrices gives us Lemma \ref{lemma:Ftcovering}.


Lastly, we give an upper bound on $\mathfrak{R}_S(\cF_{\gamma, t})$ in the following lemma.
\begin{lemma}\label{lemma:Radcomplexity}
Under Assumptions \ref{assump1}--\ref{assump3}, given activation operators and samples $S = \{(x_{i, t}, z_{i, t})_{t=1}^T, i = 1, \dots, m\}$, the empirical Rademacher complexity $\mathfrak{R}_S(\cF_{\gamma, t})$ satisfies
\begin{align}
\mathfrak{R}_S(\cF_{\gamma, t}) = O \left(d \min\Big\{b\sqrt{d},  \rho_h B_x B_W\frac{(\rho_hB_U)^t - 1}{\rho_h B_U - 1}\Big\} \times \frac{\rho_y B_V \sqrt{\log \big(t \sqrt{dm} \frac{(\rho_hB_U)^t - 1}{\rho_h B_U - 1}\big)}}{\sqrt{m}\gamma} \right). \notag
\end{align}
\end{lemma}

The detailed proof is provided in Appendix \ref{pf:lemmaRadcomplexity}. Our proof exploits the Lipschitz continuity of $\cM$ and $\ell_\gamma$, and uses Dudley's entropy integral as the standard machinery to establish Lemma \ref{lemma:Radcomplexity}. Combining Lemma \ref{lemma:startpoint} and Lemma \ref{lemma:Radcomplexity}, we complete the proof.



\section{Refined Generalization Bounds}\label{sec:fast}

When additional norm constraints on weight matrices $U, V$ and $W$ are available, we can further refine generalization bounds. 
Specifically, we consider assumptions as follows.
\begin{assumption}\label{assump7}
The weight matrices satisfy $\lVert U \rVert_{2, 1} \leq M_U$, $\lVert V \rVert_{2, 1} \leq M_V$, and $\lVert W \rVert_{2, 1} \leq M_W$.
\end{assumption}

\begin{assumption}\label{assump8}
The weight matrices satisfy $\lVert U \rVert_{\textrm{F}} \leq B_{U, \textrm{F}}$, $\lVert V \rVert_{\textrm{F}} \leq B_{V, \textrm{F}}$, and $\lVert W \rVert_{\textrm{F}} \leq B_{W, \textrm{F}}$.
\end{assumption}

Note that Assumption \ref{assump7} appears in \cite{bartlett2017spectrally} and Assumption \ref{assump8} appears in \cite{neyshabur2017pac}. We have an equivalent relation between matrix norms, i.e., $\norm{\cdot}_2 \leq \norm{\cdot}_{2, 1} \leq \sqrt{d} \norm{\cdot}_{\textrm{F}} \leq d \norm{\cdot}_2$. Comparing to Assumption \ref{assump2}, Assumptions \ref{assump7} and \ref{assump8} further restrict the model class. We then establish refined empirical Rademacher complexities for vanilla RNNs, the corresponding generalization bounds follows immediately.


\begin{theorem}\label{thm:tightbound}
Let activation operators $\sigma_h$ and $\sigma_y$ be given, and Assumptions \ref{assump1}-\ref{assump3} hold. Then for $(x_t, z_t)_{t=1}^T$ and $S = \left\{(x_{i, t}, z_{i, t})_{t=1}^T, i = 1, \dots, m \right\}$ drawn i.i.d. from any underlying distribution over $\RR^{d_x \times T} \times \{1, \dots, K\}$, with probability at least $1 - \delta$ over $S$, for every margin value $\gamma > 0$ and every $f_t \in \cF_t$ for integer $t \leq T$, the following two bounds hold:

$\bullet$ Suppose Assumption \ref{assump7} also holds. We have
\begin{align}
\mathfrak{R}_S(\cF_{\gamma, t}) = O \left(\frac{t \alpha S_{2, 1} \frac{(\rho_h B_U)^t - 1}{\rho_hB_U - 1} \sqrt{\log d} \log (\sqrt{d}m)}{\sqrt{m}\gamma} \right), \label{eq:21bound}
\end{align}
where $\alpha = \rho_h^2 \rho_y B_V B_W B_x$, $S_{2, 1} = \left(M_U + M_V + M_W\right)$, and $d =  \sqrt{d_xd_h+d_h^2+ d_hd_y}$.

$\bullet$ Suppose Assumption \ref{assump8} also holds. We have
\begin{align}
\mathfrak{R}_S(\cF_{\gamma, t}) = O \left(\frac{\alpha^\prime B_{U} \min\big\{b \sqrt{d}, \rho_h B_x B_W \frac{(\rho_h B_U)^t - 1}{\rho_hB_U - 1}\big\} S_{\textrm{F}} \frac{(\rho_h B_U)^t - 1}{\rho_hB_U - 1} \sqrt{d \ln d } }{\sqrt{m} \gamma} \right), \label{eq:pacbound}
\end{align}
where $\alpha^\prime = \rho_h \rho_y B_WB_x$, $S_{\textrm{F}} = B_{U,\textrm{F}} + B_{W,\textrm{F}} + B_{V,\textrm{F}}$, and $d = \sqrt{d_xd_h+d_h^2+ d_hd_y}$.
\end{theorem}

The detailed proof is provided in Appendix \ref{pf:thmtightbound}. The first bound \eqref{eq:21bound} adapts the matrix covering lemma in \cite{bartlett2017spectrally}. The second bound \eqref{eq:pacbound} adapts the PAC-Bayes approach \citep{neyshabur2017pac} by analyzing the divergence when imposing small perturbations on the weight matrices. 

We highlight the improvements of the obtained refined generalization bounds: When the weight matrices are approximately low rank, that is, $\norm{\cdot}_{2, 1} \ll d \norm{\cdot}_2$ and $\norm{\cdot}_\textrm{F} \ll \sqrt{d} \norm{\cdot}_2$, for $\beta \leq 1$, bound \eqref{eq:pacbound} improves bound \eqref{eq:genbound} by reducing dependence on $d$. Additionally, if $t \left(M_U+M_V+M_W\right) < d$, bound \eqref{eq:21bound} also tightens bound \eqref{eq:genbound}. Note that $t \left(M_U+M_V+M_W\right) < d$ implies that the input sequence is relatively short.

\section{Extensions to MGU, LSTM, and Conv RNNs}\label{sec:extension}

We extend our analysis to Minimal Gated Unit (MRU), Long Short-Term Memory (LSTM) RNNs and Convolutional RNNs (ConvRNNs).
%

The MGU RNNs compute 
\begin{align}
r_t = \sigma (W_r x_t + U_r h_{t-1}), \quad \tilde{h}_t = \sigma_h\left(W_h x_t + U_h (r_t \odot h_{t-1})\right), \quad h_t = (1 - r_t) \odot h_{t-1} + r_t \odot \tilde{h}_t,\notag
\end{align}
where $W_r, W_h \in \RR^{d_h \times d_x}$, $U_r, U_h \in \RR^{d_h \times d_h}$, $V \in \RR^{d_y \times d_h}$, and $r_t \in \RR^{d_h}$. The notation $\odot$ denotes the Hadamard product (entrywise product) of vectors. Denote by $\cF_{g, t}$ the class of mappings from the first $t$ inputs to the $t$-th output computed by gated (MGU or LSTM) RNNs.
For simplicity, we consider $\sigma$ being the sigmoid function, i.e., $\sigma(x) = (1 + \exp(-x))^{-1}$, $\sigma_h(\cdot) = \tanh(\cdot)$, and $\sigma_y$ being $\rho_y$-Lipschitz with $\sigma_y(0) = 0$. Extensions to general Lipschitz activation operators as in Assumption \ref{assump3} are straightforward. Suppose we have $h_0 = 0$ and the following assumption.

\begin{figure}[!htb]
\begin{center}
	\includegraphics[width=0.36\textwidth]{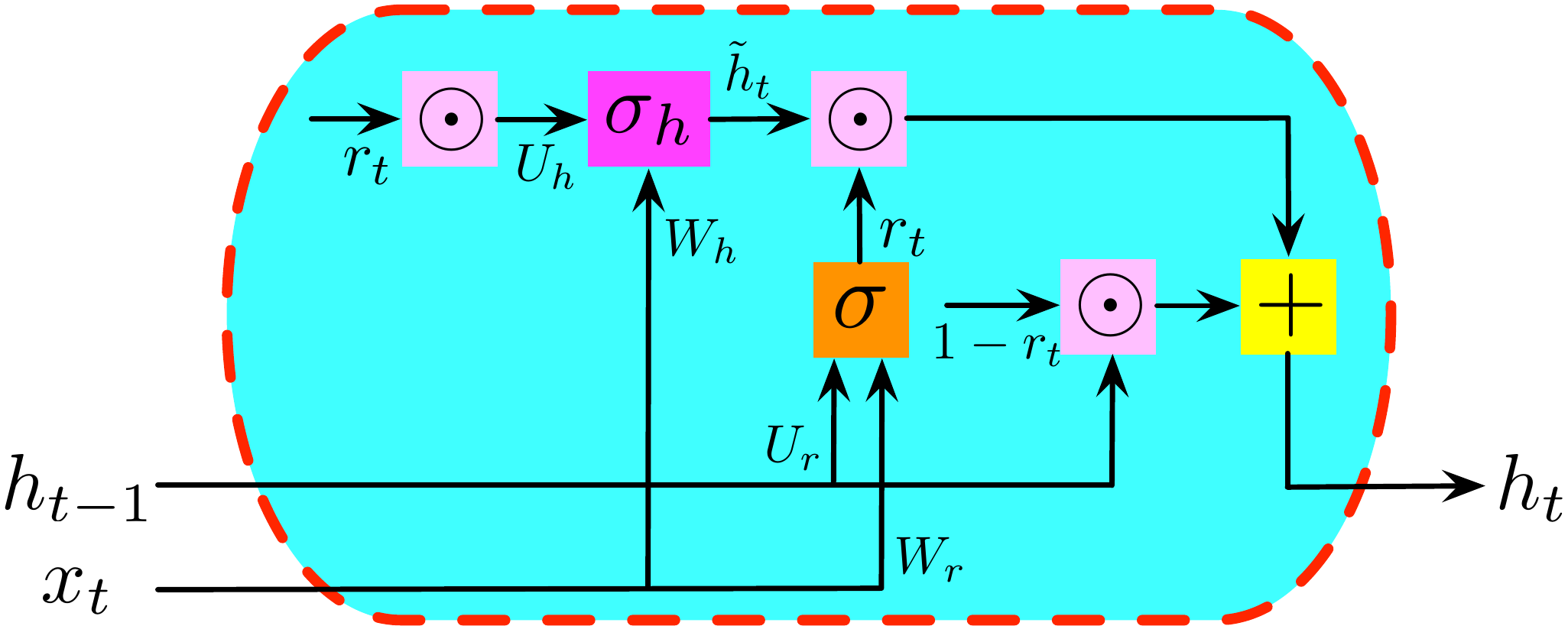}
\end{center}
\caption{A building block of MGU RNNs.}
\end{figure}

\begin{assumption}\label{assump4}
All the weight matrices have bounded spectral norms respectively, i.e. $\lVert W_r \rVert_2 \leq B_{W_r},  \lVert W_h \rVert_2 \leq B_{W_h}, \lVert U_r \rVert_2 \leq B_{U_r}, \lVert U_h \rVert_2 \leq B_{U_h},$ and $\lVert V \rVert_2 \leq B_V.$
\end{assumption} 
A similar argument for vanilla RNNs yields a generalization bound of MGU RNNs as follows.
\begin{theorem}\label{thm:MGUgenbound}
Let the activation operator $\sigma_y$ be given and Assumptions \ref{assump1} and \ref{assump4} hold. Then for $(x_t, z_t)_{t=1}^T$ and $S = \left\{(x_{i, t}, z_{i, t})_{t=1}^T, i = 1, \dots, m \right\}$ drawn i.i.d. from any underlying distribution over $\RR^{d_x \times T} \times \{1, \dots, K\}$, with probability at least $1 - \delta$ over $S$, for every margin value $\gamma > 0$ and every $f_t \in \cF_{g, t}$ for integer $t \leq T$, we have
\begin{align}
\PP\left(\tilde{z}_t \neq z_t \right) \leq \hat{\cR}_\gamma(f_t) + 3\sqrt{\frac{\log \frac{2}{\delta}}{2m}} + O \left(\frac{d \rho_y B_V \min \big\{\sqrt{d}, B_{W_h}B_x\frac{\beta^t-1}{\beta-1}\big\} \sqrt{\log \big(\frac{\theta^t - 1}{\theta - 1} d\sqrt{m} \big)}}{\sqrt{m}\gamma} \right), \notag
\end{align}
where $\beta = \max_{j}\big\{\left\lVert 1 - r_j \right\rVert_\infty +  B_{U_h} \left \lVert r_j \right\rVert_\infty^2 \big\}$, $\theta = \beta + 2B_{U_r}  + B_{U_r} B_{U_h}$ and $d= \max\{d_x, d_y, d_h\}$.
\end{theorem}
The detailed proof is provided in Appendix \ref{pf:thmMGUgenbound}. As can be seen, $r_t$ shrinks the magnitude of hidden state to reduce the dependence on $d$ and $t$ in generalization. 
As a result, with proper normalization of weight matrices, the generalization bound of MGU RNNs is less dependent on $d, t$.
%

The LSTM RNNs are more complicated than MGU RNNs, which introduce more gates to control the information flow in RNNs. LSTM RNNs have two hidden states, and compute them as,
\begin{align}
	g_t  &= \sigma (W_g x_t + U_g h_{t-1}), \quad r_t = \sigma (W_r x_t + U_r h_{t-1}), \notag \\
	o_t & = \sigma (W_o x_t + U_o h_{t-1}), \quad \tilde{c}_t = \sigma_c \left(W_c x_t + U_c h_{t-1}\right), \notag \\
	c_t & = g_t \odot c_{t-1} + r_t \odot \tilde{c}_t, \quad h_t = o_t \odot \tanh(c_t), \notag
\end{align}
where $W_g, W_r, W_o, W_c \in \RR^{d_h \times d_x}$, $U_g, U_r, U_o, U_c \in \RR^{d_h \times d_h}$, and $g_t, r_t, o_t \in \RR^{d_h}$. For simplicity, we also consider $\sigma$ being the sigmoid function, and $\sigma_c(\cdot) = \tanh(\cdot)$. The $t$-th output is $y_t = \sigma_y(V h_t)$, where $V \in \RR^{d_y \times d_h}$, and $\sigma_y$ is $\rho_y$-Lipschitz with $\sigma_y(0) = 0$. Suppose we have $h_0 = c_0 = 0$ and the following assumption.

\begin{figure}[!htb]
\begin{center}
	\includegraphics[width=0.38\textwidth]{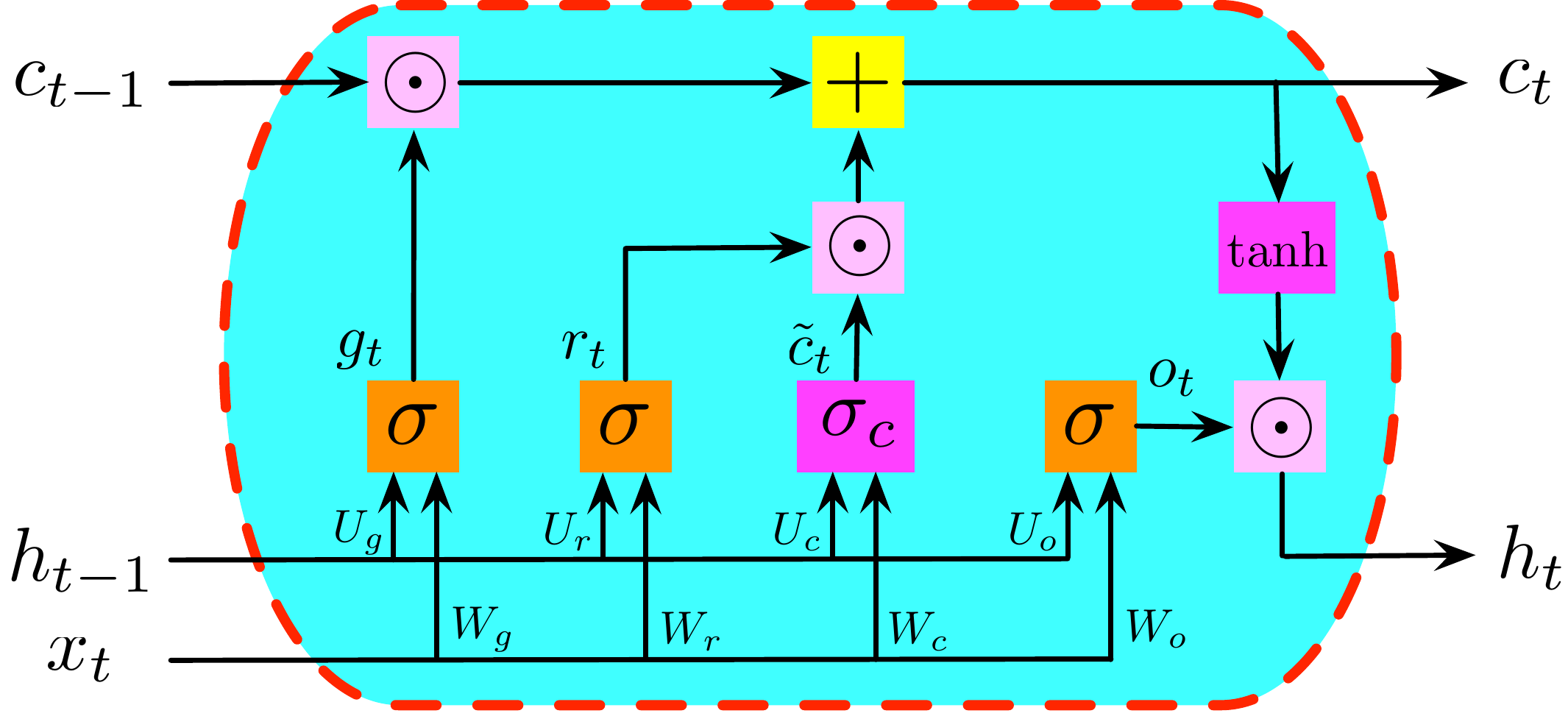}
\end{center}
\caption{A building block of LSTM RNNs.}
\end{figure}

\begin{assumption}\label{assump5} 
The spectral norms of weight matrices are bounded respectively, i.e. $\lVert W_g \rVert_2 \leq B_{W_g}, \lVert W_r \rVert_2 \leq B_{W_r}, \lVert W_o \rVert_2 \leq B_{W_o}, \lVert W_c \rVert_2 \leq B_{W_c}, \lVert U_g \rVert_2 \leq B_{U_g}, \lVert U_r \rVert_2 \leq B_{U_r}, \lVert U_o \rVert_2 \leq B_{U_o}, \lVert U_h \rVert_2 \leq B_{U_h},$ and $\lVert V \rVert_2 \leq B_V.$
\end{assumption}
For properly normalized weight matrices $W_o$ and $U_o$, the generalization bound of LSTM RNNs is given in the following theorem.
\begin{theorem}\label{thm:LSTMgenbound}
Let the activation operator $\sigma_y$ be given and Assumptions \ref{assump1} and \ref{assump5} hold. Then for $(x_t, z_t)_{t=1}^T$ and $S = \left\{(x_{i, t}, z_{i, t})_{t=1}^T, i = 1, \dots, m \right\}$ drawn i.i.d. from any underlying distribution over $\RR^{d_x \times T} \times \{1, \dots, K\}$, with probability at least $1 - \delta$ over $S$, for every margin value $\gamma > 0$ and every $f_t \in \cF_{g, t}$ for integer $t \leq T$, we have
\begin{align}
\PP\left(\tilde{z}_t \neq z_t \right) \leq \hat{\cR}_\gamma(f_t) + 3\sqrt{\frac{\log \frac{2}{\delta}}{2m}} + O \left(\frac{d \rho_y B_V \min\big\{\sqrt{d}, B_{W_c}B_x\frac{\beta^t-1}{\beta-1}\big\}\sqrt{\log \big(\frac{\theta^t - 1}{\theta - 1} d\sqrt{m} \big)}}{\sqrt{m}\gamma}\right), \notag
\end{align}
where $\beta = \max\left\{\left\lVert g_j \right\rVert_\infty +  B_{U_c} \left \lVert r_j \right\rVert_\infty \left \lVert o_j \right\rVert_\infty \right\}$, $\theta = \beta + B_{U_g} + B_{U_r} + B_{U_o}$, and $d = \max\{d_x, d_y, d_h\}$.
\end{theorem}
The detailed proof is provided in Appendix \ref{pf:thmLSTMgenbound}. Similar to MGU RNNs, LSTM RNNs also introduce extra decaying factors to reduce the dependence on $d$ and $t$ in generalization. However, LSTM RNNs are more complicated, but more flexible than MGU RNNs, since three factors, $r_t$, $o_t$ and $g_t$ are used to jointly control the spectrum of $U_c$. We further remark that LSTM RNNs need spectral norms of weight matrices, $W_g, W_r, W_o, U_g, U_r$ and $U_o$, to be properly controlled for obtaining better generalization bounds.




We further extend our analysis to Convolutional RNNs (Conv RNNs). Conv RNNs integrate convolutional filters and recurrent neural networks. Specifically, we consider input $x \in \RR^d$ and $k$-channel $k$-dimensional convolutional filters $\cI_1, \dots, \cI_k \in \RR^k$ followed by an average pooling layer over the $k$ channels for reducing dimensionality. Extensions to convolution with strides and other kinds of average pooling layers (e.g., blockwise pooling) are straightforward.

Here we denote the circulant-like matrix generated by $\cI_i$ as
\begin{align}
C_i =
\begin{bmatrix}
\cI_i^\top \quad~~ \underbrace{0 \dots \dots \dots 0}_{d-k} \\
0  \quad~~ \cI_i^\top \quad \underbrace{0 \dots \dots 0}_{d-k-1}\\
~ ~~~\quad \quad \quad \quad \ddots \quad \quad~~~ ~ \\
\underbrace{0 \dots \dots \dots 0}_{d-k} \quad~~ \cI_i^\top
\end{bmatrix} \in \RR^{(d-k+1) \times d}, \notag
\end{align} 
and write $W_{\cI} = [C_1^\top, \dots, C_k^\top]^\top$. We further denote $P = \frac{1}{k} \underbrace{\left[I_{d-k+1}~ I_{d-k+1}~ \cdots ~I_{d-k+1}\right]}_{\textrm{totally}~k~ \textrm{identity matrices}},$ where $I_d$ denotes the $d$-dimensional identity matrix. Define $\cI = [\cI_1, \dots, \cI_k]$, and $\cI \ast x = PW_\cI x$. Given a sample $(x_t, z_t)_{t=1}^T$, the Conv RNNs compute $h_t$ and $y_t$ as follows,
$$h_{t} = \sigma_h\left(\cU \ast h_{t-1} + \cW \ast x_{t}\right),\quad\textrm{and}\quad y_{t} = \sigma_y\left(\cV\ast h_{t}\right),$$ where $h_t, x_t \in \RR^d$, and $\cU, \cV, \cW \in \RR^{k \times k}$ are matrices with column vectors being $k$-dimensional convolutional filters. We use zero-padding to ensure the output dimension of convolutional filters matches the input \citep{krizhevsky2012imagenet}. To get $y_t$, we convolve $h_t$ with $\cV$ followed by an average pooling to reduce the dimension to $K$. Since we aim to show that Conv RNNs reduce the dependence on $d$ in generalization through parameter sharing, we simplify the notations to assume $h_0 = 0$, and impose the following assumption. Extensions to general settings are straightforward.
\begin{figure}[!htb]
\begin{center}
\includegraphics[width=0.5\textwidth]{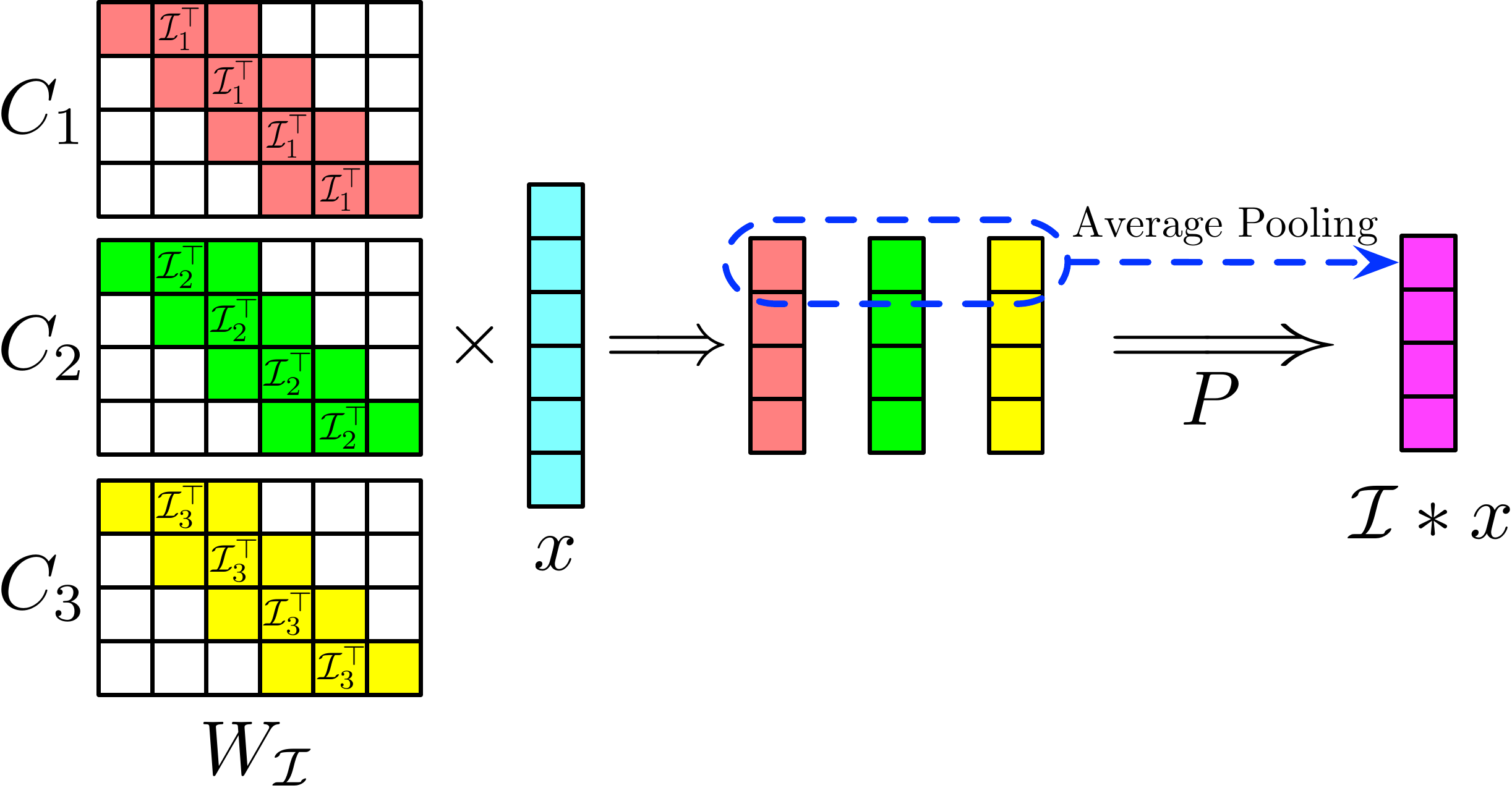}
\end{center}
\caption{\it Illustration of input $x \in \RR^6$ convolving with 3-channel 3-dimensional convolutional filters $\cI_1, \cI_2$, and $\cI_3$, followed by an average pooling.}
\end{figure}
\begin{assumption}\label{assump6}
The activation operators $\sigma_h$ and $\sigma_y$ are 1-Lipschitz with $\sigma_h(0) = \sigma_y(0) = 0$. $\sigma_h$ is entrywise bounded by 1. The convolutional filters $\cU$, $\cV$, and $\cW$ are orthogonal with normalized columns, i.e., $\cU^\top \cU = \cU \cU^\top = \frac{1}{k} I_k, \cV^\top \cV = \cV \cV^\top = \frac{1}{k} I_k,$ and $\cW^\top \cW = \cW \cW^\top = \frac{1}{k} I_k.$
\end{assumption}
We remark that the orthogonality constraints enhance the diversity among convolutional filters \citep{xie2017all, huang2017orthogonal}. Additionally, the normalization factor $\frac{1}{k}$ is to control the spectral norms of $W_{\cU}$, $W_\cV$, and $W_\cW$, which prevents the blowup of hidden state. Denote by $\cF_{c, t}$ the class of mappings from the first $t$ inputs to the $t$-th output computed by Conv RNNs. Then the generalization bound is given in the following theorem.
\begin{theorem}\label{thm:Convgenbound}
Let activation operators $\sigma_h$ and $\sigma_y$ be given, and Assumptions \ref{assump1} and \ref{assump6} hold. Then for $(x_t, z_t)_{t=1}^T$ and $S = \left\{(x_{i, t}, z_{i, t})_{t=1}^T, i = 1, \dots, m \right\}$ drawn i.i.d. from any underlying distribution over $\RR^{d \times T} \times \{1, \dots, K\}$, with probability at least $1 - \delta$ over $S$, for every margin value $\gamma > 0$ and every $f_t \in \cF_{c, t}$ for integer $t \leq T$, we have
\begin{align}
\PP\left(\tilde{z}_t \neq z_t \right) \leq \hat{\cR}_\gamma(f_t) + O\left(\frac{k t \sqrt{\log \left(dt \sqrt{m} \right)}}{\sqrt{m}\gamma} + \sqrt{\frac{\log \frac{1}{\delta}}{m}}\right). \notag
\end{align}
\end{theorem}
The detailed proof is provided in \ref{pf:thmConvgenbound}. Similar to the analysis of vanilla RNNs, our proof is based on the Lipschitz continuity of Conv RNNs with respect to its model parameters in the convolutional filters. Specifically, by Assumption \ref{assump6}, the spectral norms of $W_\cU$, $W_\cV$, and $W_\cW$ are all bounded by 1. Combining with the inequality, $\norm{W_\cU}_\textrm{F} \leq \sqrt{d} \norm{\cU}_\textrm{F}$, we have $\lVert y_t - y'_t \rVert_2 \leq L_{V, t} \lVert \cV - \cV' \rVert_\textrm{F} + L_{\cU, t} \lVert \cU - \cU' \rVert_\textrm{F} + L_{\cW, t} \lVert \cW - \cW' \rVert_\textrm{F},$ where $L_{\cU, t}$, $L_{\cV, t}$, and $L_{\cW, t}$ are polynomials in $d$ and $t$. Additionally, observe that the total number of parameters in a Conv RNN is at most $3k^2$, which is independent of input dimension $d$. As a consequence, the generalization bound of Conv RNNs only has a lieanr dependence on $k$ and $t$.


\vspace{-0.05in}
\section{Numerical Evaluation}\label{sec:experiment}
\vspace{-0.05in}

We demonstrate a comparison among our obtained generalization bound with \cite{bartlett2017spectrally}, \cite{neyshabur2017pac}, and \cite{zhang2018stabilizing}. Specifically, we train\footnote{We adopt code: \url{https://github.com/pytorch/examples/tree/master/word_language_model}.} a vanilla RNN on the wikitext language modeling dataset \citep{merity2016pointer}. We take $\sigma_h = \tanh$ and set the hidden state $h \in \RR^{128}$ and the input $x \in \RR^{14}$ with $\norm{x}_2 \leq 1$. Accordingly, we have $d = 128$ and take the sequence length $t = 56$. We list the complexity bounds for vanilla RNNs in Theorem \ref{thm:genbound} (Ours), \cite{zhang2018stabilizing} (Bound 1), \eqref{eq:21bound} of Theorem \ref{thm:tightbound} (Bound 2), and \eqref{eq:pacbound} of Theorem \ref{thm:tightbound} (Bound 3) neglecting common log factors in $d$ and $t$:
\begin{itemize}
\item Ours: $d B_V \min\big\{\sqrt{d}, B_W\frac{B_U^t - 1}{B_U - 1}\big\} \sqrt{\log \big(\frac{B_U^t - 1}{B_U - 1}\big)}$;

\item Bound 1: $dt^2 B_V B_W \max\{1, B_U^{t}\}$;

\item Bound 2: $B_V B_W \left(M_U + M_V + M_W\right)t \frac{B_U^t - 1}{B_U - 1}$;

\item Bound 3: $\big(\min\{\sqrt{d}, B_W \frac{B_U^t - 1}{B_U - 1}\} B_{U} + B_{W}\big) \frac{B_U^t - 1}{B_U - 1} \times \sqrt{d (B_{U,\textrm{F}}^2 + B_{W,\textrm{F}}^2 + B_{V,\textrm{F}}^2)}.$
\end{itemize}
The corresponding complexity bounds are shown in Figure \ref{fig:boundcompare}. As can be seen, our complexity bound in Theorem \ref{thm:genbound} is much smaller than Bounds 1-3. In more detail, the trained vanilla RNN has $B_U = 2.6801 > 1$. As discussed earlier, for $B_U > 1$, only our bound in Theorem \ref{thm:genbound} is polynomial in the size of the network, while Bounds 1-3 are all exponential in $t$. The resulting complexity bounds corroborate such a conclusion.

\begin{figure}[!htb]
\centering
\includegraphics[width = 0.35\textwidth]{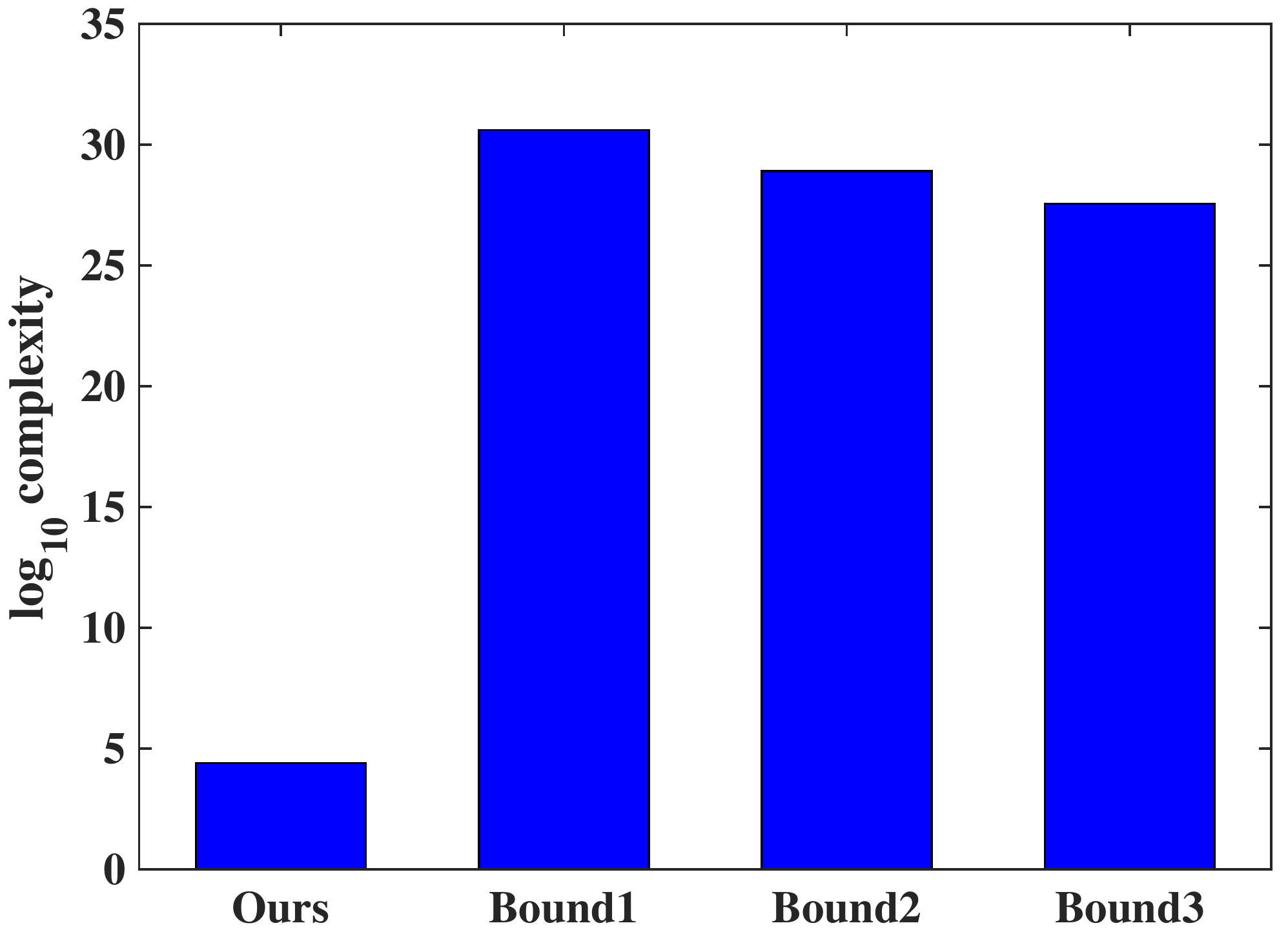}
\caption{Complexity bounds on wikitext dataset.}
\label{fig:boundcompare}
\end{figure}

We also observe that Bound 3 is smaller than Bound 2. The reason behind is that the weight matrices in the trained vanilla RNN have relatively small Frobenius norms but large $(2, 1)$ norms. Taking matrix $U$ as an example, we have $B_{U, \textrm{F}} = 13.6823$ and $M_U = 154.5439$. Then, we can calculate the stable rank $\frac{B_{U, \textrm{F}}}{B_U} = 5.1 < \sqrt{d} / 2$, and the ratio $\frac{M_{U}}{B_{U, \textrm{F}}} = 11.3 \approx \sqrt{d}$. This implies that the singular values of $U$ are not evenly distributed, while the norms of row vectors in $U$ are approximately equal.


\vspace{-0.05in}
\section{Discussions and Open Problems}\label{sec:discussion}
\vspace{-0.05in}

\noindent{\bf (I) Tighter bounds:} Our obtained generalization bounds depend on the spectral norms of weight matrices and the network size. Can we exploit other modeling structures to further reduce the dependence on the network size? Or can we find better choices of norms of weight matrices that yield better bounds?


\noindent{\bf (II) Margin value:} Our generalization bounds depend on the margin value of the predictors. As can be seen, a larger margin value yields a better generalization bound. However, establishing a sharp characterization of the margin value is technically very challenging, because of its complicated dependence on the underlying data distribution and the training algorithm. 


\noindent{\bf (III) Implicit bias of SGD:} Numerous empirical evidences have already shown that RNNs trained by stochastic gradient descent (SGD) algorithms have superior generalization performance. There have been a few theoretical results showing that SGD tends to yield low complexity models, which can generalize \citep{neyshabur2014search, neyshabur2015path, zhang2016understanding, soudry2017implicit}. Can we extend this argument to RNNs? For example, can SGD always yield weight matrices with well controlled spectra? This is crucial to the generalization of MGU and LSTM RNNs.



\noindent{\bf (IV) Adaptivity to the underlying distribution:} The current PAC-Learning framework focuses on the worst case. Taking classification as an example, the theoretical analysis holds even when the input features and labels are completely independent. Therefore, this often yields very pessimistic results. For many real applications, however, data are not obtained adversarially. Some recent empirical evidences suggest that the generalization of neural networks seems very adaptive to the underlying distribution: Easier tasks lead to low complexity neural networks, while harder ones lead to highly complex neural networks. Unfortunately, none of the existing analysis can take the underlying distribution into consideration.

\noindent{\bf (V) Sequentially dependent data:} To extend the analysis to scenarios where input sequences are dependent is quite challenging and largely open. \cite{rakhlin2015online} propose a so-called ``Sequential Rademacher Complexity'' to quantify the model complexity with dependent data. Their bound however, is exponential in the depth of a neural network, even with proper normalization on the weight matrices. \cite{kuznetsov2017generalization} also derive generalization bounds for dependent data under mixing conditions. They assume block independence for a sub-sample selection trick. The extension to fully dependent data is beyond the scope of this paper. We leave it for future investigation.



%
%
%
%

\bibliography{ref}
\bibliographystyle{ims}


\newpage
\onecolumn
\appendix

\section{Proofs in Section \ref{sec:theory}}\label{pf:sectheory}
\subsection{Lipschitz Continuity of $\cM$ and $\ell_\gamma$}
We show the Lipschitz continuity of the margin operator $\cM$ and the loss function $\ell_\gamma$ in the following lemma.
\begin{lemma}\label{lemma:MLLip}
The margin operator $\cM$ is 2-Lipschitz in its first argument with respect to vector Euclidean norm, and $\ell_\gamma$ is $\frac{1}{\gamma}$-Lipschitz.
\end{lemma} 
\begin{proof}
Let $y$, $y'$ and $z$ be given, then
\begin{align}
\bigg|\cM(y, z) - \cM\left(y', z\right)\bigg| & = \bigg| y_z - y'_z + \left(\max_{j \neq z} y'_j - \max_{j \neq z} y_j \right)\bigg| \notag \\
& \leq \bigg| y_z - y'_z \bigg| + \bigg| \max_{j \neq z} y'_j - y_j \bigg| \notag \\
& \leq 2 \left \lVert y - y' \right \rVert_\infty \leq \left\lVert y - y' \right\rVert_2. \notag
\end{align}
For function $\ell_\gamma$, it is a piecewise linear function. Thus, it is straightforward to see that $\ell_\gamma$ is $\frac{1}{\gamma}$-Lipschitz.
\end{proof}

\subsection{Proof of Lemma \ref{lemma:ytLipUVW}}
\begin{proof}\label{pf:lemmaytLipUVW}
The Lemma is stated with matrix Frobenius norms. However, we can show a tighter bound only involving the spectral norms of weight matrices. Given weight matrices $U, V, W$ and $U', V', W'$, consider the $t$-th outputs $y_t$ and $y'_t$ of vanilla RNNs,
\begin{align}
\left \lVert y_t - y'_t \right\rVert_2 & = \left\lVert \sigma_y(Vh_t) - \sigma_y(V'h'_t) \right\rVert_2 \notag \\
& \leq \rho_y \left\lVert Vh_t - V'h_t + V'h_t - V'h'_t \right\rVert_2 \notag \\
& \leq \rho_y \left(\left\lVert (V - V') h_t \right\rVert_2 + \left \lVert V'(h_t - h'_t)\right \rVert_2\right) \notag \\
& \leq \rho_y \left(\left\lVert h_t \right\rVert_2 \left \lVert V - V' \right \rVert_2 + B_V \left \lVert h_t - h'_t \right\rVert_2 \right). \label{eq:yLip}
\end{align}

We have to bound the norm of $h_t$ as in the following lemma.
\begin{lemma}\label{lemma:htbound}
Under Assumptions \ref{assump1} to \ref{assump3}, for $t \geq 0$, the norm of $h_t$ is bounded by
\begin{align}
\left\lVert h_t \right\rVert_2 \leq \min\left\{b\sqrt{d}, \rho_h B_W B_x \frac{\left(\rho_hB_U\right)^t - 1}{\rho_h B_U - 1}\right\}. \label{eq:htbound}
\end{align}
\end{lemma}
\begin{proof}
We prove by induction. Observe that for $t \geq 1$, we have
\begin{align}
\left\lVert h_t \right\rVert_2 & = \left\lVert \sigma_h(Wx_t + U h_{t-1}) \right\rVert_2 \notag \\
& \leq \rho_h \left\lVert Wx_t + U h_{t-1} \right\rVert_2 \notag \\
& \leq \rho_h \left(\left\lVert Wx_t \right\rVert_2 + \left \lVert U h_{t-1} \right\rVert_2 \right) \notag \\
& \leq \rho_h \left(B_W B_x + B_U \left\lVert h_{t-1} \right\rVert_2 \right). \label{eq:htrecursive}
\end{align}
Applying equation \eqref{eq:htrecursive} recursively with $h_0 = 0$, we arrive at,
\begin{align}
\left\lVert h_t \right\rVert_2 \leq \rho_h B_W B_x \sum_{j=0}^{t-1} (\rho_h B_U)^j = \rho_h B_W B_x \frac{\left(\rho_hB_U\right)^t - 1}{\rho_h B_U - 1}, \notag
\end{align}
We also have $\lVert h_t \rVert_\infty \leq b$. Thus, combining with the above upper bound, we get $$\left\lVert h_t \right\rVert_2 \leq \min\left\{b\sqrt{d}, \rho_h B_W B_x \frac{\left(\rho_hB_U\right)^t - 1}{\rho_h B_U - 1}\right\}.$$ Clearly, $\lVert h_0 \rVert_2 = 0$ satisfies the upper bound.
\end{proof}

When $\rho_h B_U = 1$, the ratio is defined, by L'Hospital's rule, to be the limit,
\begin{align}
\lim_{\rho_hB_U \rightarrow 1} \frac{(\rho_hB_U)^t - 1}{\rho_hB_U - 1} = t. \notag
\end{align}

With Lemma \ref{lemma:htbound} in hand, we plug the bound \eqref{eq:htbound} into equation \eqref{eq:yLip} and end up with
\begin{align}\label{eq:yLipVh}
\left \lVert y_t - y'_t \right\rVert_2 \leq \rho_y \rho_h B_W B_x \frac{\left(\rho_hB_U\right)^t - 1}{\rho_h B_U - 1} \left \lVert V - V' \right \rVert_2 + \rho_y B_V \left \lVert h_t - h'_t \right\rVert_2.
\end{align}

The remaining task is to bound $\left \lVert h_t - h'_t \right\rVert_2$ in terms of the spectral norms of the difference of weight matrices, $\left\lVert W - W' \right\rVert_2$ and $\left\lVert U - U' \right\rVert_2$.

\begin{lemma}\label{lemma:htLipWU}
Under Assumptions \ref{assump1} to \ref{assump3}, for $t \geq 1$, the difference of hidden states $h_t$ and $h'_t$ satisfies
\begin{align*}
\left \lVert h_t - h'_t \right\rVert_2 \leq L_{W, t} \left\lVert W - W' \right\rVert_2 + L_{U, t} \left\lVert U - U' \right\rVert_2,
\end{align*}
where 
$L_{W, t} = \rho_h B_x \frac{(\rho_h B_U)^t - 1}{\rho_h B_U - 1}$ and $L_{U, t} = \rho_h^2 B_W B_x t \frac{(\rho_hB_U)_2^t - 1}{(\rho_hB_U) - 1}$.
\end{lemma}
\begin{proof}
Similar to the proof of Lemma \ref{lemma:htbound}, we use induction.
\begin{align}
\left \lVert h_t - h'_t \right\rVert_2 & = \left \lVert \sigma_h\left(W x_t + U h_{t-1}\right) - \sigma_h\left(W'x_t + U' h'_{t-1} \right) \right\rVert_2 \notag \\
& \leq \rho_h \left\lVert (W - W')x_t + Uh_{t-1} - U'h'_{t-1} \right\rVert_2 \notag \\
& \leq \rho_h \left(\left\lVert (W - W')x_t \right\rVert_2 + \left \lVert Uh_{t-1} - U'h'_{t-1} \right\rVert_2\right) \notag \\
& \leq \rho_h \left(B_x \left\lVert W - W' \right\rVert_2 + \left \lVert Uh_{t-1} - U'h_{t-1} + U'h_{t-1} - U'h'_{t-1} \right\rVert_2 \right) \notag \\
& \leq \rho_h B_x \left\lVert W - W' \right\rVert_2 + \rho_h \left(\left\lVert h_{t-1} \right\rVert_2 \left\lVert U - U' \right\rVert_2 + B_U \left \lVert h_{t-1} - h'_{t-1} \right\rVert_2 \right). \notag
\end{align}
Repeat this derivation recursively, we have
\begin{align}
\left \lVert h_t - h'_t \right\rVert_2 & \leq \rho_h B_x \left\lVert W - W' \right\rVert_2 + \rho_h \left\lVert h_{t-1} \right\rVert_2 \left\lVert U - U' \right\rVert_2 + \rho_h B_U \left \lVert h_{t-1} - h'_{t-1} \right\rVert_2 \notag \\
& \leq \rho_h B_x \left(1 + \rho_h B_U \right) \left\lVert W - W' \right\rVert_2 + \rho_h \left(\left\lVert h_{t-1} \right\rVert_2 + \rho_h B_U \left\lVert h_{t-2} \right\rVert_2\right)  \left\lVert U - U' \right\rVert_2 \nonumber\\
&\hspace{0.2in}+ \left(\rho_h B_U \right)^2 \left \lVert h_{t-2} - h'_{t-2} \right\rVert_2 \notag \\
& \leq \dots \dots \notag \\
& \leq \rho_h B_x \sum_{j=0}^{t-1} \left(\rho_h B_U \right)^j \left\lVert W - W' \right\rVert_2 + \rho_h \sum_{j=0}^{t-1} \left((\rho_h B_U)^{t-1-j} \left\lVert h_j\right\rVert_2 \right) \left\lVert U - U' \right\rVert_2 \nonumber\\
&\hspace{0.2in}+ (\rho_h B_U)^t \left \lVert h_0 - h'_0 \right\rVert_2 \notag \\
& \leq \rho_h B_x \frac{(\rho_h B_U)^t - 1}{\rho_h B_U - 1} \left\lVert W - W' \right\rVert_2 + \rho_h \sum_{j=0}^{t-1} \left((\rho_h B_U)^{t-1-j} \left\lVert h_j\right\rVert_2 \right) \left\lVert U - U' \right\rVert_2. \label{eq:htLipWU}
\end{align}
We now plug in the upper bound \eqref{eq:htbound} to calculate the summation involving the Euclidean norms of the hidden state $h_t$.
\begin{align}
\sum_{j=0}^{t-1} (\rho_h B_U)^{t-1-j} \left\lVert h_j\right\rVert_2 & \leq \sum_{j=0}^{t-1} (j+1)(\rho_h B_U)^{j} \rho_h B_W B_x \leq t \sum_{j=0}^{t-1} (\rho_hB_U)^j \rho_h B_W B_x \notag \\
& \leq \rho_h B_W B_x t \frac{(\rho_hB_U)^t - 1}{\rho_hB_U - 1}. \notag
\end{align}
Plugging back into equation \eqref{eq:htLipWU}, we have as desired,
\begin{align}
\left \lVert h_t - h'_t \right\rVert_2 \leq \rho_h B_x \frac{(\rho_h B_U)^t - 1}{\rho_h B_U - 1} \left\lVert W - W' \right\rVert_2 + \rho_h^2 B_W B_x t \frac{(\rho_hB_U)^t - 1}{\rho_hB_U - 1} \left\lVert U - U' \right\rVert_2. \notag
\end{align}
\end{proof}

Combining equation \eqref{eq:yLipVh} and Lemma \ref{lemma:htLipWU}, and $\norm{W}_{\textrm{F}} \geq \norm{W}_2$, we immediately get Lemma \ref{lemma:ytLipUVW}. 
\end{proof}

\subsection{Proof of Lemma \ref{lemma:Ftcovering}}
\begin{proof}\label{pf:lemmaFtcovering}
Our goal is to construct a covering $\cC(\cF_t, \epsilon, \textrm{dist}(\cdot, \cdot))$, i.e., for any $f_t \in \cF_t$, there exists $\hat{f}_t \in \cF_t$, for any input data $(x_{t})_{t=1}^T$, satisfying
\begin{align}
\sup_{X_t} \left\lVert f_t(X_t) - \hat{f}_t(X_t) \right\rVert_2 \leq \epsilon. \notag
\end{align}
Note that $f$ is determined by weight matrices $U, V$ and $W$. By Lemma \ref{lemma:ytLipUVW}, we have
\begin{align}
\sup_{X_t} \left \lVert f(X_t) - \hat{f}_t(X_t) \right \rVert_2 \leq L_{V, t} \left\lVert V - \hat{V} \right\rVert_\textrm{F} + L_{W, t} \left\lVert W - \hat{W} \right\rVert_\textrm{F} + L_{U, t} \left \lVert U - \hat{U} \right\rVert_\textrm{F}. \notag
\end{align}
Then it is enough to construct three matrix coverings, $\cC\left(U, \frac{\epsilon}{3 L_{U, t}}, \lVert \cdot \rVert_\textrm{F}\right)$, $\cC\left(V, \frac{\epsilon}{3 L_{V, t}}, \lVert \cdot \rVert_\textrm{F}\right)$ and $\cC\left(W, \frac{\epsilon}{3 L_{W, t}}, \lVert \cdot \rVert_\textrm{F}\right)$. Their Cartesian product gives us the covering $\cC(\cF_t, \epsilon, \textrm{dist}(\cdot, \cdot))$. The following lemma gives an upper bound on the covering number of matrices with a bounded Frobenius norm.
\begin{lemma}\label{lemma:matrixcovering}
Let $\cG = \left\{A \in \RR^{d_1 \times d_2} : \lVert A \rVert_2 \leq \lambda \right\}$ be the set of matrices with bounded spectral norm and $\epsilon > 0$ be given. The covering number $\cN(\cG, \epsilon, \lVert \cdot \rVert_\textrm{F})$ is upper bounded by
\begin{align}
\cN(\cG, \epsilon, \lVert \cdot \rVert_\textrm{F}) \leq \left(1 + 2\frac{\min\left\{\sqrt{d_1}, \sqrt{d_2}\right\} \lambda}{\epsilon}\right)^{d_1d_2}. \notag
\end{align}
\end{lemma}
\begin{proof}
For any matrix $A \in \cG$, we define a mapping $\phi : \RR^{d_1 \times d_2} \mapsto \RR^{d_1d_2}$, such that $\phi(A) = [A_{:, 1}^\top, A_{:, 2}^\top, \dots, A_{:, h}^\top]^\top$, where $A_{:, i}$ denotes the $i$-th column of matrix $A$. Denote the vector space induced by the mapping $\phi$ by $\cV(\cG) = \left\{\phi(A) : A \in \cG\right\}$. Note that we have $\lVert A \rVert_{\textrm{F}}^2 = \sum_{i=1}^h A_{:, i}^\top A_{:, i} = \lVert \phi(A) \rVert_2^2$ and the mapping $\phi$ is one-to-one and onto. By definition, the square of Frobenius norm equals the square of sum of singular values and the spectral norm is the largest singular value. Hence, the equivalence of Frobenius norm and spectral norm is given by the following inequalities,
\begin{align}
\lVert A \rVert_2 \leq \lVert A \rVert_\textrm{F} \leq \min\left\{\sqrt{d_1}, \sqrt{d_2}\right\} \lVert A \rVert_2. \notag
\end{align}
Now, we see that if we construct a covering $\cC(\cV(\cG), \epsilon, \lVert \cdot \rVert_2)$, then $$\phi^{-1}\cC(\cV(\cG), \epsilon, \lVert \cdot \rVert_2) = \left\{\phi^{-1}(v) : v \in \cC(\cV(\cG), \epsilon, \lVert \cdot \rVert_2)\right\}$$ is a covering of $\cG$ at scale $\epsilon$ with respect to the matrix Frobenius norm. Therefore, we get $$\cN(\cG, \epsilon, \lVert \cdot \rVert_\textrm{F}) \leq \cN(\cV(\cG), \epsilon, \lVert \cdot \rVert_2).$$ As a consequence, it is suffices to upper bound the covering number of $\cV(\cG)$. In order to do so, we need another closely related concept, packing number.
\begin{definition}[Packing]
Let $\cG$ be an arbitrary set and $\epsilon > 0$ be given. We say $\cP(\cG, \epsilon, \lVert \cdot \rVert)$ is a packing of $\cG$ at scale $\epsilon$ with respect to the norm $\lVert \cdot \rVert$, if for any two elements $A, B \in \cP$, we have
\begin{align}
\left\lVert A - B \right\rVert > \epsilon. \notag
\end{align}
Denote by $\cM(\cG, \epsilon, \lVert \cdot \rVert)$ the maximal cardinality of $\cP(\cG, \epsilon, \lVert \cdot \rVert)$.
\end{definition}
By the maximality, we can check that $\cN(C, \epsilon, \lVert \cdot \rVert) \leq \cM(C, \epsilon, \lVert \cdot \rVert)$. Indeed, let $\cP^*(\cG, \epsilon, \lVert \cdot \rVert)$ be a maximal packing. Suppose there exists $A \in \cG$ such that for any $B \in \cP^*(\cG, \epsilon, \lVert \cdot \rVert)$, the inequality $\left\lVert A - B \right\rVert > \epsilon$ holds. Then we can add $A$ to $\cP^*(\cG, \epsilon, \lVert \cdot \rVert)$, while still keeping it being a packing, which contradicts the maximality of $\cP^*(\cG, \epsilon, \lVert \cdot \rVert)$. Thus, we have $\cN(\cG, \epsilon, \lVert \cdot \rVert) \leq \cM(\cG, \epsilon, \lVert \cdot \rVert)$.

Observe that $\cV(\cG)$ is contained in an Euclidean ball $\cB(0; R) \in \RR^{d_1 d_2}$ of radius at most 
\begin{align}
R = \max_{A \in \cG} \lVert \phi(A) \rVert_2 \leq \min\left\{\sqrt{d_1}, \sqrt{d_2}\right\} \lVert A \rVert_2 \leq \min\left\{\sqrt{d_1}, \sqrt{d_2}\right\} \lambda. \notag
\end{align}
Additionally, the union of Euclidean balls $\cB(v; \epsilon/2) \subset \RR^{d_1d_2}$ with radius $\epsilon/2$ and center $v \in \cP(\cV(\cG), \epsilon, \lVert \cdot \rVert_2)$ is further contained in an Euclidean ball $\cB(0; R_\epsilon)$ of slightly enlarged radius $R_\epsilon = \min\left\{\sqrt{d_1}, \sqrt{d_2}\right\} \lambda + \epsilon/2$. Those balls $\cB(v; \epsilon/2)$ are disjoint by the definition of packing, thus we have
\begin{align}
\cN(\cV(C), \epsilon, \lVert \cdot \rVert_2) &\leq \cP(\cV(C), \epsilon, \lVert \cdot \rVert_2) \leq \frac{\textrm{vol}(\cB(0, R_\epsilon))}{\textrm{vol}(\cB(v; \epsilon/2))} = \left(\frac{R_\epsilon}{\epsilon/2}\right)^{d_1d_2} \nonumber\\
&= \left(1 + 2\frac{\min\{\sqrt{d_1}, \sqrt{d_2}\} \lambda}{\epsilon}\right)^{d_1d_2}, \notag
\end{align}
where $\textrm{vol}(\cdot)$ denotes the volume.
\end{proof}
By Lemma \ref{lemma:matrixcovering}, we can directly write out the upper bounds on the covering numbers of weight matrices,
\begin{align}
& \cN\left(U, \frac{\epsilon}{3 L_{U, t}}, \lVert \cdot \rVert_\textrm{F}\right) \leq \left(1 + 6\frac{\sqrt{d_h} B_U L_{U, t}}{\epsilon}\right)^{d_h^2}, \notag \\
& \cN\left(V, \frac{\epsilon}{3 L_{V, t}}, \lVert \cdot \rVert_\textrm{F}\right) \leq \left(1 + 6\frac{\min\{\sqrt{d_y}, \sqrt{d_h}\} B_V L_{V, t}}{\epsilon}\right)^{d_y d_h}, \quad\textrm{and} \notag \\
& \cN\left(W, \frac{\epsilon}{3 L_{W, t}}, \lVert \cdot \rVert_\textrm{F}\right) \leq \left(1 + 6\frac{\min\{\sqrt{d_x}, \sqrt{d_h}\} B_W L_{W, t}}{\epsilon}\right)^{d_x d_h}. \notag
\end{align}
Then we immediately have,
\begin{align}
\cN(\cF_t, \epsilon, \textrm{dist}(\cdot, \cdot)) & \leq \cN\left(U, \frac{\epsilon}{3 L_{U, t}}, \lVert \cdot \rVert_\textrm{F}\right) \times \cN\left(V, \frac{\epsilon}{3 L_{V, t}}, \lVert \cdot \rVert_\textrm{F}\right) \times \cN\left(W, \frac{\epsilon}{3 L_{W, t}}, \lVert \cdot \rVert_\textrm{F}\right) \notag \\
& \leq \left(1 + \frac{6\sqrt{d_h} B_U L_{U, t}}{\epsilon}\right)^{d_h^2} \left(1 + \frac{\min\{6\sqrt{d_y}, \sqrt{d_h}\} B_V L_{V, t}}{\epsilon}\right)^{d_y d_h} \notag\\
& ~~ \times \left(1 + \frac{6\min\{\sqrt{d_x}, \sqrt{d_h}\} B_W L_{W, t}}{\epsilon}\right)^{d_x d_h}. \notag
\end{align}
Substituting the coefficients $L_{U, t}, L_{V, t}$ and $L_{W, t}$ from Lemma \ref{lemma:ytLipUVW}, we get 
\begin{align}
&\cN(\cF_t, \epsilon, \textrm{dist}(\cdot, \cdot)) \nonumber\\
& \leq \left(1 + \frac{6\sqrt{d} \rho_y\rho_hB_VB_WB_x \frac{(\rho_hB_U)^t - 1}{\rho_h B_U - 1}}{\epsilon}\right)^{2d^2} \left(1 + \frac{6\sqrt{d} \rho_y\rho_h^2B_UB_VB_WB_x t \frac{(\rho_hB_U)^t - 1}{\rho_h B_U - 1}}{\epsilon}\right)^{d^2} \notag \\
& \leq \left(1 + \frac{6 c \sqrt{d} t \frac{(\rho_hB_U)^t - 1}{\rho_h B_U - 1}}{\epsilon}\right)^{3d^2}, \notag
\end{align}
where $c = \rho_y\rho_hB_VB_WB_x \max\left\{1, \rho_hB_U\right\}$. For future usage, we also write down for small $\epsilon > 0$, such that $\frac{6 c \sqrt{d} t \frac{(\rho_hB_U)^t - 1}{\rho_h B_U - 1}}{\epsilon} > 1$, the logarithm of covering number satisfies,
\begin{align}
\log \cN(\cF_t, \epsilon, \textrm{dist}(\cdot, \cdot)) \leq 3d^2 \log \left(\frac{12 c \sqrt{d} t \frac{(\rho_hB_U)^t - 1}{\rho_h B_U - 1}}{\epsilon}\right). \notag
\end{align} 
\end{proof}

\subsection{Proof of Lemma \ref{lemma:Radcomplexity}}
\begin{proof}\label{pf:lemmaRadcomplexity}
Define $\cF_{\cM, t} = \left\{(X_t, z_t) \mapsto \cM(f_t(X_t), z_t) : f_t \in \cF_t \right\}$. By Lemma \ref{lemma:MLLip}, we see that $\cM$ is 2-Lipschitz in its first argument. In order to cover $\cF_{\cM, t}$ at scale $\epsilon$, it suffices to cover $\cF_t$ at scale $\frac{\epsilon}{2}$. This immediately gives us the covering number $\cN(\cF_{\cM, t}, \epsilon, \lVert \cdot \rVert_\infty) \leq \cN(\cF_t, \epsilon/2, \textrm{dist}(\cdot, \cdot))$. 

We then give the statement of Dudley's entropy integral.
\begin{lemma}\label{lemma:dudley}
Let $\cH$ be a real-valued function class taking values in $[-r, r]$ for some constant $r$, and assume that $0 \in \cH$. Let $S = (s_1, \dots, s_m)$ be given points, then
\begin{align}
\mathfrak{R}_S(\cH) \leq \inf_{\alpha > 0} \left(\frac{4\alpha}{\sqrt{m}} + \frac{12}{m} \int_\alpha^{2r\sqrt{m}} \sqrt{\log \cN(\cH, \epsilon, \lVert \cdot \rVert)} d\epsilon \right). \notag
\end{align}
\end{lemma}
The proof can be found in \cite{bartlett2017spectrally}. Taking $\cH = \cF_{\cM, t}$, we can easily verify that $\cF_{\cM, t}$ takes values in $[-r, r]$ with $r = \rho_y B_V \lVert h_t \rVert_2 \leq \rho_y B_V \min\left\{b\sqrt{d}, \rho_h B_W B_x \frac{(\rho_h B_U)^t - 1}{\rho_h B_U - 1}\right\}$ and $0 \in \cF_\cM$. Thus, directly applying Lemma \ref{lemma:dudley} yields the following bound,
\begin{align}
\mathfrak{R}_S(\cF_{\cM, t}) \leq \inf_{\alpha > 0} \left(\frac{4\alpha}{\sqrt{m}} + \frac{12}{m} \int_\alpha^{2r\sqrt{m}} \sqrt{\log \cN(\cF_{\cM, t}, \epsilon, \lVert \cdot \rVert_\infty)} d\epsilon \right). \notag
\end{align}
We bound the integral as follows,
\begin{align}
\int_\alpha^{2r\sqrt{m}} \sqrt{\log \cN(\cF_{\cM, t}, \epsilon, \lVert \cdot \rVert_\infty)} d\epsilon & \leq \int_\alpha^{2r\sqrt{m}} \sqrt{3d^2 \log \left(\frac{24 c \sqrt{d} t \frac{(\rho_hB_U)^t - 1}{\rho_h B_U - 1}}{\epsilon}\right)} d\epsilon \notag \\
& \leq 2r\sqrt{m} \sqrt{3d^2 \log \left(\frac{24 c \sqrt{d} t \frac{(\rho_hB_U)^t - 1}{\rho_h B_U - 1}}{\alpha}\right)}. \notag
\end{align}
Picking $\alpha = \frac{1}{\sqrt{m}}$ is enough to give us an upper bound on $\mathfrak{R}_S(\cF_{\cM, t})$,
\begin{align}
\mathfrak{R}_S(\cF_\cM) \leq \frac{4}{m} + \frac{24}{\sqrt{m}} \sqrt{3d^2 r^2 \log \left(24 c \sqrt{dm} t \frac{(\rho_hB_U)^t - 1}{\rho_h B_U - 1}\right)}. \notag
\end{align}
Finally, by Talagrand's lemma \citep{mohri2012foundations} and $\cL_\gamma$ being $\frac{1}{\gamma}$-Lipschitz, we have
\begin{align}
\mathfrak{R}_S(\cF_{\gamma, t}) \leq \frac{1}{\gamma} \mathfrak{R}_S(\cF_{\cM, t}) \leq \frac{4}{m\gamma} + \frac{24}{\sqrt{m}\gamma} \sqrt{3d^2 r^2 \log \left(24 c \sqrt{dm} t \frac{(\rho_hB_U)^t - 1}{\rho_h B_U - 1}\right)}. \notag
\end{align}
\end{proof}

\section{Proof in Section \ref{sec:fast}}
\subsection{Proof of Theorem \ref{thm:tightbound}}
\begin{proof}\label{pf:thmtightbound}
Under additional Assumption \ref{assump7}, we only need to show that, with the additional matrix induced norm bound, we have a refined upper bound on the matrix covering number. The proof relies on the following lemma adapted from \cite{bartlett2017spectrally} Lemma 3.2.
\begin{lemma}
Let $\cG = \left\{A \in \RR^{d_1 \times d_2} : \norm{A}_{2, 1} \leq \lambda \right\}$. We have the following matrix covering upper bound
\begin{align}
\log \cN (\cG, \epsilon, \norm{\cdot}_2) \leq \frac{\lambda^2}{\epsilon^2} \log (2d_1 d_2). \notag
\end{align}
\end{lemma}
The above Lemma is a direct consequence of Lemma 3.2 in \cite{bartlett2017spectrally} with $X$ being identity, $a = \lambda$, $b = 1$, and $m = d_1, d = d_2$. We apply the same trick to split the overall covering accuracy $\epsilon$ into 3 parts, $\frac{\epsilon}{3 L_{U, t}}$, $\frac{\epsilon}{3 L_{V, t}}$, and $\frac{\epsilon}{3 L_{W, t}}$, corresponding to $U, V, W$ respectively. Then we derive a refined bound on the covering number of $\cF_t$:
\begin{align}\label{eq:tightcover}
\log \cN(\cF_t, \epsilon, \textrm{dist}(\cdot, \cdot)) \leq \frac{9\left(M_UL^2_{U, t} + M_VL^2_{V, t} + M_WL^2_{W, t}\right)}{\epsilon^2} \log (2d^2),
\end{align} 
where $d = \max\left\{d_x, d_y, d_h\right\}$. Substituting \eqref{eq:tightcover} into the Dudley integral as in the proof of Lemma \ref{lemma:Radcomplexity}  yields
\begin{align}
\mathfrak{R}_S(\cF_{\cM, t}) \leq \inf_{\alpha > 0} \left(\frac{4\alpha}{\sqrt{m}} + \frac{12}{m} \int_\alpha^{2r\sqrt{m}} \sqrt{\log \cN(\cF_{t}, \epsilon/2, \lVert \cdot \rVert_\infty)} d\epsilon \right). \notag
\end{align} 
We bound the integral as follows,
\begin{align}
\int_\alpha^{2r\sqrt{m}} \sqrt{\log \cN(\cF_{t}, \epsilon/2, \lVert \cdot \rVert_\infty)} d\epsilon & \leq \int_\alpha^{2r\sqrt{m}} 36\frac{\sqrt{M_UL^2_{U, t} + M_VL^2_{V, t} + M_WL^2_{W, t}}}{\epsilon} \sqrt{\log (2d^2)} d\epsilon \notag \\
& = 36 \sqrt{M_UL^2_{U, t} + M_VL^2_{V, t} + M_WL^2_{W, t}} \sqrt{\log (2d^2)} \log \frac{2r\sqrt{m}}{\alpha}. \notag
\end{align}
Choosing $\alpha = \frac{1}{\sqrt{m}}$ yields
\begin{align}
\mathfrak{R}_S(\cF_\cM) \leq \frac{4}{m} + \frac{432}{\sqrt{m}} \sqrt{M_UL^2_{U, t} + M_VL^2_{V, t} + M_WL^2_{W, t}} \sqrt{\log (2d^2)} \log \left(2m\sqrt{d}\right). \notag
\end{align}
Finally, substituting the Lipschitz constant $L_{U, t}, L_{V, t}, L_{W, t}$ into the expression, we have
\begin{align}
\mathfrak{R}_S(\cF_{\gamma, t}) & \leq \frac{1}{\gamma} \mathfrak{R}_S(\cF_{\cM, t}) \leq \frac{4}{m\gamma} + \frac{432}{\gamma \sqrt{m}} \sqrt{M_UL^2_{U, t} + M_VL^2_{V, t} + M_WL^2_{W, t}} \sqrt{\log (2d^2)} \log \left(2m\sqrt{d}\right) \notag \\
& \leq O\left(\frac{\alpha \max\{M_U, M_V, M_W\} t \frac{\beta^t - 1}{\beta - 1}}{\gamma \sqrt{m}} \sqrt{\log d} \log \left(m\sqrt{d}\right) \right). \notag
\end{align}
Combining with Lemma \ref{lemma:startpoint} completes the proof.

Under additional Assumption \ref{assump8}, our proof is based on the following result from Lemma 1 in \cite{neyshabur2017pac}. 
\begin{lemma}\label{lem:pac_bayes_bd}
Let $f_{\alpha}\rbr{x}: \cX \rightarrow \RR^d$ be any predictor with parameter $\alpha$, and $\cP$ be any distribution on the parameter that is independent of training data. Then, for any $\gamma, \delta > 0$, with probability at least $1-\delta$ over the training set of size $m$, for any $\alpha$ and any random perturbation $\beta$ s.t. $\PP_{\beta} \sbr{\max_{x \in \cX} \abr{f_{\alpha+\beta}\rbr{x} - f_{\alpha}\rbr{x} }_{\infty} < \frac{\gamma}{4} } \geq \frac{1}{2}$, we have\begin{align*}
\cR_{0}\rbr{f_{\alpha}} - \hat{\cR}_{\gamma}\rbr{f_{\alpha}} \leq 4 \sqrt{\frac{ {\rm KL}\rbr{\alpha+\beta \| \cP } + \log \rbr{\frac{6m}{\delta}} }{m-1} },
\end{align*}
where ${\rm KL}\rbr{\alpha+\beta \| \cP }$ is KL divergence of distributions $\alpha+\beta$ and $\cP$.
\end{lemma}

For convenience, we omit the superscript for sample index. Denote $h_{t} \rbr{\alpha}$ and $h_{t} \rbr{\alpha+\beta}$ as the hidden variables with parameters $\alpha$ and $\alpha+\beta$ respectively. Then we provide an upper bound of the gap of hidden layers before and after the perturbation. Denote the parameters $\alpha = \textrm{vec}\rbr{\cbr{W,U,V}}$ and the perturbation $\beta = \textrm{vec}\rbr{\cbr{\delta W,\delta U,\delta V}}$.

For any $t \in \{1, 2, \dots, T\}$, we have
\begin{align}
&\nbr{h_{t} \rbr{\alpha+\beta} - h_{t} \rbr{\alpha}}_2 \nonumber \\
\overset{(i)}{\leq} &~ \rho_h \nbr{\rbr{U+\delta U} h_{t-1} \rbr{\alpha+\beta} + \rbr{W+\delta W}x_{t} - U h_{t-1} \rbr{\alpha} - W x_{t} }_2 \nonumber \\
\overset{(ii)}{\leq} &~ \rho_hB_{U} \nbr{h_{t-1} \rbr{\alpha+\beta} - h_{t-1} \rbr{\alpha}}_2 + \delta  \rho_h B_{U} \nbr{h_{t-1} \rbr{\alpha+\beta}}_2 + \delta \rho_h B_x B_{W}   \nonumber \\
\leq &~ (\rho_h B_{U})^t \nbr{h_{0} \rbr{\alpha+\beta} - h_{0} \rbr{\alpha}}_2 + \delta \sum_{i=1}^{t}  (\rho_h B_{U})^i  \nbr{h_{t-i} \rbr{\alpha+\beta}}_2 + \delta \rho_h B_x B_{W} \sum_{i=0}^{t-1}  (\rho_h B_{U})^i , \label{eqn:pac_bayes_h_gap_bd}
\end{align}
By Lemma \ref{lemma:htbound}, we have that for any $t \leq T$,
\begin{align}
\nbr{h_{t} \rbr{\alpha} }_2 \leq \min \cbr{ b \sqrt{p}, \rho_h B_x  B_{W} \frac{ (\rho_h B_{U})^t - 1}{ \rho_h B_{U} - 1}} = \lambda_{t}. \label{eqn:h_bd}
\end{align}

Combining \eqref{eqn:pac_bayes_h_gap_bd}, \eqref{eqn:h_bd}, and $h_{0}=0$, we have
\begin{align}
\nbr{h_{t} \rbr{\alpha+\beta} - h_{t} \rbr{\alpha}}_2 & \leq \delta \lambda_{t} \sum_{i=1}^{t}  (\rho_h B_{U})^i  + \delta  \rho_h B_x B_{W} \sum_{i=0}^{t-1}  (\rho_h B_{U})^i \nonumber \\
&\leq \delta \rbr{ \lambda_{t} \rho_h B_{U} + \rho_h B_x B_{W}} \frac{(\rho_h B_{U})^t - 1}{\rho_hB_{U}-1}. \label{eqn:h_gap_bd}
\end{align}

Denote $y_{t} \rbr{\alpha}$ and $y_{t} \rbr{\alpha+\beta}$ as the out with parameters $\alpha$ and $\alpha+\beta$ respectively. Then we have
\begin{align}
\nbr{y_{t} \rbr{\alpha+\beta} - y_{t} \rbr{\alpha}}_2 & \overset{(i)}{\leq} \rho_y \nbr{\rbr{1+\delta}Vh_{t} \rbr{\alpha+\beta} - V h_{t} \rbr{\alpha} }_2 \nonumber \\
&\leq  \rho_y B_{V} \nbr{h_{t} \rbr{\alpha+\beta} - h_{t} \rbr{\alpha}}_2 + \delta \rho_y B_{V} \nbr{h_{t} \rbr{\alpha+\beta}}_2 \nonumber \\
&\overset{(ii)}{\leq} \delta \rho_y B_{V} \rbr{ \lambda_t \rho_h B_{U} + \rho_h B_x B_{W}} \frac{(\rho_hB_{U})^t - 1}{\rho_hB_{U}-1} + \delta \rho_y B_{V} \lambda_t, \label{eqn:z_gap_bd}
\end{align}
where $(i)$ is from Lipschitz continuity of $\sigma_y$ and $(ii)$ is from \eqref{eqn:h_bd} and \eqref{eqn:h_gap_bd}.

Then choosing the prior distribution and the perturbation distribution as $\cN\rbr{0,\sigma^2 I}$, and from the concentration result for the spectral norm bounds, we have
\begin{align*}
\PP_{A \sim \cN\rbr{0,\sigma^2 I_{d \times d}}}\sbr{\nbr{A}_2 > \xi } \leq 2 p \exp\rbr{\frac{-\xi^2}{2d \sigma^2}}.
\end{align*}
This implies with probability at least $1/2$, we have $\max \cbr{\delta B_{U}, \delta B_{W}, \delta B_{V}} \leq \sigma\sqrt{2d \ln\rbr{12d} }$. Taking $\sigma = \rbr{{\gamma }/{ 4 \rho_y \rbr{ \rbr{ \lambda_t \rho_h B_{U} + \rho_h B_x B_{W}} \frac{(\rho_h B_{U})^t - 1}{\rho_h B_{U}-1} + \lambda_t } \sqrt{2d \ln\rbr{12d} } }}$ and combining with \eqref{eqn:z_gap_bd}, with probability at least $1/2$, we have
\begin{align*}
&\max_{x \in \cX_{m}} \nbr{y_{t} \rbr{\alpha+\beta} - y_{t} \rbr{\alpha}}_2 \\
\leq &~\rbr{ \rbr{ \lambda_t \rho_h B_{U} + \rho_h B_x B_{W}} \frac{(\rho_h B_{U})^t - 1}{\rho_h B_{U}-1} + \lambda_t } \cdot \sigma\sqrt{2d \ln\rbr{12d} } \leq \frac{\gamma}{4}.
\end{align*}
Finally, we calculate the KL divergence of $\cP$ and $\alpha+\beta$ with respect to this choice of $\sigma$,
\begin{align}
&{\rm KL}\rbr{\alpha+\beta \| \cP } \leq \frac{\nbr{\alpha}_2^2}{2\sigma^2} \nonumber \\
=&~ O\rbr{ \frac{\rho_y^2}{\gamma^2} \cdot \rbr{ \rbr{ \lambda_t \rho_h B_{U} + \rho_h B_x B_{W}} \frac{(\rho_h B_{U})^t - 1}{\rho_h B_{U}-1} + \lambda_t }^2 d \ln\rbr{d} \rbr{{ B_{U,\textrm{F}}^2} + { B_{W,\textrm{F}}^2} + { B_{V,\textrm{F}}^2}} } \nonumber \\
=&~ O \rbr{ \frac{ \rho_y^2 \rbr{ \lambda_t \rho_h B_{U} + \rho_h B_x B_{W}}^2 \rbr{\beta^t - 1}^2 p \ln\rbr{p} \rbr{ B_{U,\textrm{F}}^2 + B_{W,\textrm{F}}^2 + B_{V,\textrm{F}}^2 } }{\gamma^2 \rbr{\beta-1}^2 } }. \nonumber
\end{align}
We complete the proof by applying Lemma~\ref{lem:pac_bayes_bd}. 
\end{proof}

\section{Proofs in Sections \ref{sec:extension}}
\subsection{Proof of Theorem \ref{thm:MGUgenbound}}
\begin{proof}\label{pf:thmMGUgenbound}
We use the same argument from the analysis of vanilla RNNs to investigate the Lipschitz continuity of MGU RNNs. Consider $h_t$ and $h'_t$ computed by different sets of weight matrices.
\begin{align}
&\left\lVert h_t - h'_t \right\rVert_2  = \left \lVert (1 - r_t) \odot h_{t-1} + r_t \odot \tilde{h}_t - (1 - r'_t) \odot h'_{t-1} - r'_t \odot \tilde{h}'_t \right\rVert_2 \notag \\
& \leq \left\lVert (r'_t - r_t) \odot h'_{t-1}\right\rVert_2 + \left\lVert (1 - r_t) \odot (h_{t-1} - h'_{t-1}) \right\rVert_2 + \left\lVert (r_t - r'_t) \odot \tilde{h}'_t \right\rVert_2 + \left\lVert r_t \odot (\tilde{h}_t - \tilde{h}'_t) \right\rVert_2 \notag \\
& \leq \left\lVert r'_t - r_t \right\rVert_2 \left\lVert h'_{t-1}\right\rVert_\infty + \left\lVert 1 - r_t \right\rVert_\infty \left \lVert h_{t-1} - h'_{t-1} \right\rVert_2 + \left\lVert r_t - r'_t \right\rVert_2 \left\lVert \tilde{h}'_t \right\rVert_\infty + \left\lVert r_t \right\rVert_\infty \left\lVert \tilde{h}_t - \tilde{h}'_t \right\rVert_2 \notag
\end{align}
Expand the expression of $\tilde{h}_t$. Note that $r_t$ is nonnegative, and $\lVert r_t \rVert_\infty \leq 1$. Then we have $\lVert h_t \rVert_\infty \leq 1$. Additionally $\tanh(\cdot)$ is 1-Lipschitz. Thus we get
\begin{align}
&\left\lVert \tilde{h}_t - \tilde{h}'_t \right\rVert_2  \leq \left \lVert U_h(h_{t-1} \odot r_t) + W_h x_t - U'_h(h'_{t-1} \odot r'_t) - W'_h x_t\right\rVert_2 \notag \\
& \leq \left\lVert U_h(h_{t-1} \odot r_t) - U'_h(h'_{t-1} \odot r'_t) \right\rVert_2 + B_x \left\lVert W_h - W'_h \right\rVert_2 \notag \\
& \leq \left\lVert U_h - U'_h \right\rVert_2 \left\lVert h_{t-1} \odot r_t \right\rVert_2 + B_{U_h} \left\lVert r_t \right\rVert_\infty \left\lVert h_{t-1} - h'_{t-1} \right\rVert_2 + B_{U_h} \left\lVert h'_{t-1} \right\rVert_\infty \left\lVert r_t - r'_t \right\rVert_2 \nonumber\\
&\hspace{0.2in}+ B_x \left\lVert W_h - W'_h \right\rVert_2 \notag \\
& \leq \left\lVert h_t \right\rVert_2 \left\lVert U_h - U'_h \right\rVert_2 + B_{U_h} \left\lVert r_t \right\rVert_\infty \left\lVert h_{t-1} - h'_{t-1} \right\rVert_2 + B_{U_h} \left\lVert r_t - r'_t \right\rVert_2 + B_x \left\lVert W_h - W'_h \right\rVert_2. \notag
\end{align}
We have to expand $r_t - r'_t$ as follows,
\begin{align}
\lVert r_t - r'_t \rVert_2 & = \left\lVert W_r x_t + U_r h_{t-1} - W'_r x_t - U'_r h'_{t-1} \right\rVert_2 \notag \\
& \leq B_x \left\lVert W_r - W'_r \right\rVert_2 + B_{U_r} \left\lVert h_{t-1} - h'_{t-1} \right\rVert_2 + \lVert h'_{t-1} \rVert_2 \lVert U_r - U'_r \rVert_2. \notag 
\end{align}
We also need to bound $\lVert h_{t} \rVert_2$,
\begin{align}
\left\lVert h_{t}\right \rVert_2 & \leq \left\lVert 1 - r_t \right\rVert_\infty \left\lVert h_{t-1}\right \rVert_2 + \left \lVert r_t \right\rVert_\infty \left\lVert \tilde{h}_t \right\rVert_2 \notag \\
& \leq \left\lVert 1 - r_t \right\rVert_\infty \left\lVert h_{t-1}\right \rVert_2 + \left \lVert r_t \right\rVert_\infty \left(B_{W_h}B_x + B_{U_h} \left \lVert r_t \right\rVert_\infty \left\lVert h_{t-1} \right\rVert_2 \right) \notag \\
& = \left(\left\lVert 1 - r_t \right\rVert_\infty +  B_{U_h} \left \lVert r_t \right\rVert_\infty^2 \right) \left \lVert h_{t-1} \right\rVert_2 + B_{W_h} B_x, \notag \\
& \leq \max_{j \leq t}\left\{\left\lVert 1 - r_j \right\rVert_\infty +  B_{U_h} \left \lVert r_j \right\rVert_\infty^2 \right\} \left \lVert h_{t-1} \right\rVert_2 + B_{W_h} B_x. \notag
\end{align}
Applying the above inequality recursively and remember $\lVert h_t \rVert_\infty \leq 1$, we get $\left\lVert h_{t}\right \rVert_2 \leq \min\left\{\sqrt{d}, \frac{\beta^t - 1}{\beta - 1}B_{W_h} B_x\right\}$ with $\beta = \max_{j \leq t}\left\{\left\lVert 1 - r_j \right\rVert_\infty +  B_{U_h} \left \lVert r_j \right\rVert_\infty^2 \right\}$.
Put all the above ingredients together, we have
\begin{align}
\left\lVert h_t - h'_t \right\rVert_2 \leq & \left(\beta + 2B_{U_r} + B_{U_r} B_{U_h} \right) \left\lVert h_{t-1} - h'_{t-1} \right\rVert_2 \notag \\
& + \sqrt{d} \left\lVert U_h - U'_h \right\rVert_2 + B_x \left\lVert W_h - W'_h \right\rVert_2 \notag \\
& + \left(2 + B_{U_h} \right) \sqrt{d} \left\lVert U_r - U'_r \right\rVert_2 + \left(2B_x + B_{U_h}B_x \right) \left \lVert W_r - W'_r \right \rVert_2. \notag
\end{align}
Apply the above inequality recursively, denote by $\theta = \beta + 2B_{U_r} + B_{U_r}B_{U_h}$, we have
\begin{align}
\left\lVert h_t - h'_t \right\rVert_2 \leq & \sqrt{d}\sum_{j=1}^t \theta^j \left\lVert U_h - U'_h \right\rVert_2 + B_x \sum_{j=1}^t \theta^j \left\lVert W_h - W'_h \right\rVert_2 \notag \\
& + \left(2\sqrt{d} + B_{U_h} \sqrt{d} \right) \sum_{j=1}^t \theta^j \left\lVert U_r - U'_r \right\rVert_2 + \left(2B_x + B_{U_h}B_x \right) \sum_{j=1}^t \theta^j \left \lVert W_r - W'_r \right \rVert_2. \notag
\end{align}
We then derive the Lipschitz continuity of $\lVert y_t \rVert_2$,
\begin{align}
&\left\lVert y_t - y'_t \right\rVert_2  \leq \rho_y B_V \lVert h_t - h'_t \rVert_2 + \rho_y \sqrt{d} \lVert V - V' \rVert_2 \notag \\
& \leq \rho_y B_V \sqrt{d} \frac{\theta^t - 1}{\theta - 1} \left\lVert U_h - U'_h \right\rVert_2 + \rho_y B_V  B_x \frac{\theta^t - 1}{\theta - 1} \left\lVert W_h - W'_h \right\rVert_2 + \rho_y \sqrt{d} \lVert V - V' \rVert_2 \notag \\
&~~ + \rho_y B_V  \left(2\sqrt{d} + B_{U_h} \sqrt{d} \right) \frac{\theta^t - 1}{\theta - 1} \left\lVert U_r - U'_r \right\rVert_2 + \rho_y B_V  \left(2B_x + B_{U_h}B_x \right) \frac{\theta^t - 1}{\theta - 1} \left \lVert W_r - W'_r \right \rVert_2. \notag
\end{align}
Following the same argument for proving the generalization bound of vanilla RNNs, we can get the generalization bound for MGU RNNs as
\begin{align}
\PP\left(\tilde{z}_t \neq z_t \right) \leq \hat{\cR}_\gamma(f_t) + O\left(\frac{d \rho_y B_V \min\left\{\sqrt{d}, B_{W_h}B_x\frac{\beta^t-1}{\beta-1}\right\}\sqrt{\log \left(d\sqrt{m} \frac{\theta^t - 1}{\theta - 1} \right)}}{\sqrt{m}\gamma} + \sqrt{\frac{\log \frac{1}{\delta}}{m}}\right). \notag
\end{align}

\end{proof}

\subsection{Proof of Theorem \ref{thm:LSTMgenbound}}
\begin{proof}\label{pf:thmLSTMgenbound}
We first bound the norm of $h_t$ as follows,
\begin{align}
\lVert h_t \rVert_2 & \leq \lVert o_t \rVert_\infty \lVert \tanh(c_t) \rVert_2 \leq \lVert o_t \rVert_\infty \lVert c_t \rVert_2 \notag \\
& \leq \lVert g_t \rVert_\infty \lVert c_{t-1} \rVert_2 + \lVert r_t \rVert_\infty \lVert \tilde{c}_t \rVert_2 \notag \\
& \leq \lVert g_t \rVert_\infty \lVert c_{t-1} \rVert_2 + \lVert r_t \rVert_\infty \left(B_{W_c} B_x + B_{U_c} \lVert h_{t-1} \rVert_2 \right) \notag \\
& \leq \lVert g_t \rVert_\infty \lVert c_{t-1} \rVert_2 + \lVert r_t \rVert_\infty \left(B_{W_c} B_x + B_{U_c} \lVert o_t \rVert_\infty \lVert c_{t-1} \rVert_2 \right) \notag \\
& \leq \left(\lVert g_t \rVert_\infty + \lVert r_t \rVert_\infty \lVert o_t \rVert_\infty B_{U_c}\right) \lVert c_{t-1} \rVert_2 + B_{W_c} B_x. \notag
\end{align}
By applying the above inequality recursively, we have $\lVert h_t \rVert_2 \leq \lVert c_t \rVert_2 \leq B_{W_c} B_x \frac{\beta^t - 1}{\beta^t - 1}$, where $\beta = \max_{j \leq t} \left\{\lVert g_j \rVert_\infty + \lVert r_j \rVert_\infty \lVert o_j \rVert_\infty B_{U_c} \right\}$. We also have $\lVert h_t \rVert_2 \leq \sqrt{d}$. Thus, put together, we have $\lVert h_t \rVert_2 \leq \min\left\{\sqrt{d}, B_{W_c} B_x \frac{\beta^t - 1}{\beta^t - 1}\right\}$. Next, we investigate the Lipschitz continuity of $h_t$.
\begin{align}
\left\lVert h_t - h'_t \right\rVert_2 & \leq \left\lVert o_t \odot \tanh(c_t) - o'_t \odot \tanh(c'_t) \right\rVert_2 \notag \\
& \leq \lVert o_t - o'_t \rVert_2 \lVert \tanh(c_t) \rVert_\infty + \lVert o'_t \rVert_\infty \lVert c_t - c'_t \rVert_2 \notag
\end{align}
We have to expand $o_t - o'_t$,
\begin{align}
\lVert o_t - o'_t \rVert_2 \leq B_x \lVert W_o - W'_o \rVert_2 + B_{U_o} \lVert h_{t-1} - h_{t-1} \rVert_2 + \lVert h_{t-1} \rVert_2 \lVert U_o - U'_o \rVert_2. \notag
\end{align}
Note that $\lVert B_{U_o} \rVert_2$ is usually small, $o_t$ and $o'_t$ are close, and we have $\lVert h_{t-1} - h'_{t-1} \rVert_2 \leq \lVert o_t \rVert_\infty \lVert c_{t-1} - c'_{t-1} \rVert_2 \leq \lVert c_{t-1} - c'_{t-1} \rVert_2$. Thus, we can derive
\begin{align}
\lVert o_t - o'_t \rVert_2 \leq B_x \lVert W_o - W'_o \rVert_2 + B_{U_o} \lVert c_{t-1} - c_{t-1} \rVert_2 + \sqrt{d} \lVert U_o - U'_o \rVert_2. \notag
\end{align}
We also expand $c_t - c'_t$ to get,
\begin{align}
\lVert c_t - c'_t \rVert_2 \leq \lVert c_{t-1} \rVert_\infty \lVert g_t - g'_t \rVert_2 + \lVert r'_t \rVert_\infty \lVert c_{t-1} - c'_{t-1} \rVert_2 + \lVert \tilde{c}_t \rVert_\infty \lVert r_t - r'_t \rVert_2 + \lVert r'_t \rVert_\infty \lVert \tilde{c}_t - \tilde{c}'_t \rVert_2. \notag
\end{align}
We also have,
\begin{align}
& \lVert \tilde{c}_t - \tilde{c}'_t \rVert_2 \leq B_{U_c} \lVert h_{t-1} - h'_{t-1} \rVert_2 + \lVert h_{t-1} \rVert_2 \lVert U_c - U'_c \rVert_2 + B_x \lVert W_c - W'_c \rVert_2, \notag \\
& \lVert g_t - g'_t \rVert_2 \leq B_x \lVert W_g - W'_g \rVert_2 + B_{W_g} \lVert h_{t-1} - h'_{t-1} \rVert_2 + \sqrt{d}\lVert U_g - U'_g \rVert_2, \notag \quad\textrm{and} \\
&  \lVert r_t - r'_t \rVert_2 \leq B_x \lVert W_r - W'_r \rVert_2 + B_{W_r} \lVert h_{t-1} - h'_{t-1} \rVert_2 + \sqrt{d}\lVert U_r - U'_r \rVert_2. \notag
\end{align}
Putting together, we get
\begin{align}
&\lVert c_t - c'_t \rVert_2 \nonumber \\
& \leq B_x \left(\lVert W_c - W'_c \rVert_2 + \lVert W_g - W'_g \rVert_2 + \lVert W_r - W'_r \rVert_2\right) \notag \\
&\hspace{0.2in} + \sqrt{d}\left(\lVert U_c - U'_c \rVert_2 + \lVert U_g - U'_g \rVert_2 + \lVert U_r - U'_r \rVert_2 \right) \notag \\
&\hspace{0.2in} + \lVert g_t \rVert_\infty \lVert c_{t-1} - c'_{t-1} \rVert_2 + \left(\lVert r_t \rVert_\infty B_{U_c}+ B_{U_g} + B_{U_r} \right) \lVert h_{t-1} - h'_{t-1}\rVert_2 \notag \\
&\leq B_x \left(\lVert W_c - W'_c \rVert_2 + \lVert W_g - W'_g \rVert_2 + \lVert W_r - W'_r \rVert_2 + (B_{U_c}+ B_{U_g} + B_{U_r}) \lVert W_o - W'_o \rVert_2 \right) \notag \\
&\hspace{0.2in} + \sqrt{d}\left(\lVert U_c - U'_c \rVert_2 + \lVert U_g - U'_g \rVert_2 + \lVert U_r - U'_r \rVert_2 + (B_{U_c}+ B_{U_g} + B_{U_r}) \lVert U_o - U'_o \rVert_2\right) \notag \\
&\hspace{0.2in} + \left(\lVert o_t \rVert_\infty \lVert r_t \rVert_\infty B_{U_c}+ B_{U_g} + B_{U_r} + B_{U_o} \right) \lVert c_{t-1} - c'_{t-1}\rVert_2. \notag
\end{align}
By induction, we have
\begin{align}
&\lVert c_t - c'_t \rVert_2 \nonumber\\
&\leq  B_x \frac{\theta^t - 1}{\theta - 1}\left(\lVert W_c - W'_c \rVert_2 + \lVert W_g - W'_g \rVert_2 + \lVert W_r - W'_r \rVert_2 + (B_{U_c}+ B_{U_g} + B_{U_r}) \lVert W_o - W'_o \rVert_2 \right) \notag \\
&\hspace{0.2in} + \sqrt{d} \frac{\theta^t - 1}{\theta - 1} \left(\lVert U_c - U'_c \rVert_2 + \lVert U_g - U'_g \rVert_2 + \lVert U_r - U'_r \rVert_2 + (B_{U_c}+ B_{U_g} + B_{U_r}) \lVert U_o - U'_o \rVert_2\right), \notag
\end{align}
where $\theta = \beta + B_{U_g} + B_{U_r} + B_{U_o}$. Now we immediately have
\begin{align}
&\left\lVert h_t - h'_t \right\rVert_2 \nonumber\\
&\leq B_x \frac{\theta^t - 1}{\theta - 1}\left(\lVert W_c - W'_c \rVert_2 + \lVert W_g - W'_g \rVert_2 + \lVert W_r - W'_r \rVert_2 + (B_{U_c}+ B_{U_g} + B_{U_r}) \lVert W_o - W'_o \rVert_2 \right) \notag \\
&\hspace{0.2in} + \sqrt{d} \frac{\theta^t - 1}{\theta - 1} \left(\lVert U_c - U'_c \rVert_2 + \lVert U_g - U'_g \rVert_2 + \lVert U_r - U'_r \rVert_2 + (B_{U_c}+ B_{U_g} + B_{U_r}) \lVert U_o - U'_o \rVert_2\right). \notag
\end{align}
Then the Lipschitz continuity of $y_t$ can be written as
\begin{align}
\lVert y_t - y'_t \rVert_2 \leq \rho_y B_V \lVert h_t - h'_t \rVert_2 + \rho_y \sqrt{d} \lVert V - V' \rVert_2. \notag
\end{align}
Following the same argument for proving the generalization bound of vanilla RNNs, we can get the generalization bound for LSTM RNNs as
\begin{align}
\PP\left(\tilde{z}_t \neq z_t \right) \leq \hat{\cR}_\gamma(f_t) + O\left(\frac{d \rho_y B_V \min\left\{\sqrt{d}, B_{W_c}B_x\frac{\beta^t-1}{\beta-1}\right\}\sqrt{\log \left(d\sqrt{m} \frac{\theta^t - 1}{\theta - 1} \right)}}{\sqrt{m}\gamma} + \sqrt{\frac{\log \frac{1}{\delta}}{m}}\right). \notag
\end{align}
\end{proof}


\subsection{Proof of Theorem \ref{thm:Convgenbound}}
\begin{proof}\label{pf:thmConvgenbound}
We first characterize the Lipschitz continuity of $\lVert y_t \rVert_2$ with respect to model parameters $\cU$, $\cW$ and $\cV$. We have
\begin{align}
\lVert y_t - y'_t \rVert_2 \leq \rho_y \lVert h_t \rVert_2 \lVert W_\cV - W_{\cV'} \rVert_2 + \rho_y \norm{W_\cV}_2 \lVert h_t - h'_t \rVert_2. \notag
\end{align}
Since $\lVert h_t \rVert_\infty \leq 1$, we have $\lVert h_t \rVert_2 \leq \sqrt{d}$. Then we expand $h_t - h'_t$,
\begin{align}
\lVert h_t - h'_t \rVert_2 & \leq \rho_h \lVert \cU' \ast h_{t-1} + \cW \ast x_t - \cU' \ast h'_{t-1} - \cW' \ast x_t \rVert_2 \notag \\
& = \rho_h \lVert P W_\cU h_{t-1} + P W_\cW x_t - P W_\cU' h'_{t-1} - P W_\cW x_t \rVert_2 \notag \\
& \leq \rho_h \lVert P \rVert_2 \lVert W_\cU h_{t-1} + W_\cW x_t - W_{\cU'} h'_{t-1} - W_{\cW'} x_t\rVert_2 \notag \\
& \leq \rho_h \lVert P \rVert_2 \left[B_x \lVert W_\cW - W_{\cW'} \rVert_2 + \sqrt{d} \lVert W_\cU - W_{\cU'} \rVert_2 + \lVert W_\cU \rVert_2 \lVert h_{t-1} - h'_{t-1} \rVert_2 \right]. \notag
\end{align}
Observe that we have by the definition of circulant matrix,
\begin{align}
\lVert W_\cU - W_{\cU'} \rVert_2^2 \leq \lVert W_\cU - W_{\cU'} \rVert_\textrm{F}^2 = (d-k) \lVert \cU - \cU' \rVert_\textrm{F}^2 \leq d \lVert \cU - \cU' \rVert_\textrm{F}^2. \notag
\end{align}
The same holds for $W_\cW - W_{\cW'}$ and $W_\cV - W_{\cV'}$. We also have $\lVert P \rVert_2 = 1$. The remaining task is to bound the spectral norm of $W_\cU$ and $W_\cV$. Consider the matrix product $W_\cU^\top W_\cU$. We claim that the diagonal elements of $W_\cU^\top W_\cU$ is bounded by $\sum_{i=1}^k \lVert \cU_i \rVert_2^2$, and the off-diagonal elements are zero. To see this, denote by $C_{\cU_i}$ the circulant like matrix generated by $\cU_i$. Then we have $W_\cU = [C_{\cU_1}^\top, \dots, C_{\cU_k}^\top]^\top$. The diagonal elements of $W_\cU^\top W_\cU$ are
\begin{align}
\left(W_\cU^\top W_\cU\right)_{ii} = \sum_{j=1}^k \left(C_{\cU_j}^\top C_{\cU_j} \right)_{ii} \leq \sum_{i=1}^k \norm{\cU_i}_2^2. \notag
\end{align}
By the orthogonality of $\cU$, the off-diagonal elements are
\begin{align}
\left(W_\cU^\top W_\cU\right)_{pq} = \sum_{j=1}^k \left(C_{\cU_j}^\top C_{\cU_j} \right)_{pq} = \sum_{j=1}^k \left(C_{\cU_j}\right)_{:p}^\top \left(C_{\cU_j}\right)_{:q} = 0. \notag
\end{align}
Thus, the spectral norm $\norm{W_{\cU}}_2 \leq \sqrt{\sum_{i=1}^k \norm{\cU_i}_2^2} \leq 1$, and $\norm{W_{\cV}}_2, \norm{W_{\cW}}_2\leq 1$ also hold. Then we can derive
\begin{align}
\lVert h_t - h'_t \rVert_2 \leq \rho_h B_x \sqrt{d} \lVert \cW - \cW' \rVert_\textrm{F} + \rho_h d \lVert \cU - \cU' \rVert_\textrm{F} + \rho_h \lVert h_{t-1} - h'_{t-1} \rVert_2. \notag
\end{align}
Apply the above inequality recursively, we get
\begin{align}
\lVert h_t - h'_t \rVert_2 & \leq \rho_h B_x \sqrt{d} \frac{\rho_h^t - 1}{\rho_h - 1} \lVert \cW - \cW' \rVert_\textrm{F} + \rho_h d \frac{\rho_h^t - 1}{\rho_h - 1} \lVert \cU - \cU' \rVert_\textrm{F} \notag \\
& \leq B_x \sqrt{d} t \lVert \cW - \cW' \rVert_\textrm{F} + d t \lVert \cU - \cU' \rVert_\textrm{F}. \notag 
\end{align}
Thus, we have the following Lipschitz continuity of $\lVert y_t \rVert_2$,
\begin{align}
\lVert y_t - y'_t \rVert_2 \leq d \lVert \cV - \cV' \rVert_\textrm{F} + B_x \sqrt{d} t \lVert \cW - \cW' \rVert_\textrm{F} + d t \lVert \cU - \cU' \rVert_\textrm{F}. \notag
\end{align}
We also bound the norm of $h_t$ by induction. Specifically, we have
\begin{align}
\lVert h_t \rVert_2 \leq \rho_h \lVert P W_\cU h_{t-1} + P W_\cW x_t \rVert_2 \leq \rho_h \lVert W_\cU h_{t-1} \rVert_2 + \rho_h \lVert W_\cW x_t \rVert_2 \leq \norm{h_{t-1}}_2 + B_x. \notag
\end{align}
Applying the above expression recursively, we have $\lVert h_{t} \rVert_2 \leq \min\{\sqrt{d}, B_x t\} \leq B_x t$. Then following the same argument for proving the generalization bound of vanilla RNNs, we can get the generalization bound for Conv RNNs as
\begin{align}
\PP\left(\tilde{z}_t \neq z_t \right) \leq \hat{\cR}_\gamma(f_t) + O\left(\frac{B_x k t \sqrt{\log \left(dt \sqrt{m} \right)}}{\sqrt{m}\gamma} + \sqrt{\frac{\log \frac{1}{\delta}}{m}}\right). \notag
\end{align}
\end{proof}

\end{document}